\documentclass[twoside,11pt]{article}
\usepackage{blindtext}
\usepackage{algorithm}
\usepackage{algpseudocode}
\usepackage{amsmath}
\usepackage{mathtools}

\def\gg{\mathbf g }
\def\xx{\mathbf x }

\def\mm{\mathbf m }

\def\uu{\mathbf u }

\def\dd{\mathbf d}
\def\zz{\mathbf{z}}
\def\txt{\text{test}}
\def\thth{\boldsymbol{\theta} }
\def\ThTh{\Theta}

\def \XX {X}
\def \QQ {Q}
\def \EE {\mathcal{E}}
\def \XX {\mathcal{X}}
\def \GG {\mathcal{G}}
\def\PP{\mathcal P}

\def \LL {\mathcal{L}}
\def \RR {\mathbb{R}}
\def \KK {\mathcal{K}}
\def \NN {\mathcal{N}}

\def \VV {\mathcal{V}}

\def \CB {\mathcal{B}}
\def \CI {\mathcal{I}}
\def \CS {\mathcal{S}}

\newcommand{\argmax}{\operatornamewithlimits{arg\,max}}
\newcommand{\argmin}{\operatornamewithlimits{arg\,min}}

\newcommand{\yc}[1]{{{\color{red}  #1}}}

\newcommand{\rev}[1]{{{\color{black}  #1}}}

\usepackage[preprint]{jmlr2e}
\usepackage{graphicx} 
\usepackage{amssymb}
\usepackage{xcolor}
\usepackage{natbib}
\usepackage{hyperref}
 \usepackage{booktabs}
 \usepackage{chngcntr}

\newtheorem{assumption}[theorem]{Assumption}
\usepackage{lastpage}
\jmlrheading{27}{2026}{1-\pageref{LastPage}}{1/25; Revised
3/26}{6/26}{25-0217}{Sze Ming Lee and Yunxiao Chen}

\ShortHeadings{Pairwise Comparisons without Stochastic Transitivity}{Lee and Chen}
\firstpageno{1}

\begin{document}

\title{Pairwise Comparisons without Stochastic Transitivity: Model, Theory and Applications}

\author{\name Sze Ming Lee \email s.lee51@lse.ac.uk 
       \AND
       \name Yunxiao Chen \email y.chen186@lse.ac.uk \\
       \addr Department of Statistics\\
       London School of Economics and Political Science\\
        London WC2A 2AE, United Kingdom
      }

\editor{Mladen Kolar}

\maketitle

\begin{abstract}%
Most statistical models for pairwise comparisons, including the Bradley-Terry (BT) and Thurstone models and many extensions, make a relatively strong assumption of stochastic transitivity. This assumption imposes the existence of an unobserved global ranking among all the players/teams/items and monotone constraints on the comparison probabilities implied by the global ranking. 
However, the stochastic transitivity assumption does not hold in many real-world scenarios of pairwise comparisons, especially games involving multiple skills or strategies. As a result, models relying on this assumption can have suboptimal predictive performance. In this paper, we propose a general family of statistical models for pairwise comparison data without a stochastic transitivity assumption, substantially extending the BT and Thurstone models. 
In this model, the pairwise probabilities are determined by a (approximately) low-dimensional skew-symmetric matrix. Likelihood-based estimation methods and computational algorithms are developed, which allow for sparse data with only a small proportion of observed pairs. Theoretical analysis shows that the proposed estimator achieves minimax-rate optimality, which adapts effectively to the sparsity level of the data. The spectral theory for skew-symmetric matrices plays a crucial role in the implementation and theoretical analysis. The proposed method’s superiority against the BT model, along with its broad applicability across diverse scenarios, is further supported by simulations and real data analysis.
\end{abstract} 

\medskip

\section{Introduction}

Pairwise comparison data have received intensive attention in statistics and machine learning, with diverse applications across domains. Such data often arise from tournaments, where each pairwise comparison outcome results from a match between two players or teams, or from crowdsourcing settings, where individuals are tasked with comparing two items, such as images, movies, or products. Specifically, the famous Thurstone \citep{thurstone1927method} and Bradley-Terry \citep[BT;][]{bradley1952rank} models have set a cornerstone in the field, followed by many extensions, including the parametric ordinal models proposed in \cite{shah_etal-2016-JoMR}, which broadens the class of parametric models. \cite{oliveira_etal-2018-JoMR} relax the assumption of a known link function and propose models that allow the link function to belong to a broad family of functions. Nonparametric approaches have also emerged, such as the work introduced in \cite{shah_wainwright-2018-JoMR} based on the Borda counting algorithm, and the nonparametric Bradley-Terry models studied in \cite{chatterjee-2015-AoS} and \cite{chatterjee_mukherjee-2019-IEEE}. Additionally, pairwise comparison models have been developed for crowdsourced settings, as discussed in \cite{chen_Bennett_etal-2013-prooceedings} and \cite{chen_etal-2016-JoMR}, among many others.  The models for pairwise comparisons have received a wide range of applications, including ranking problems \citep{chen2015spectral,chen2019spectral,heckel2019active,chen_etal-2022-MOR,liu2023lagrangian,fan2024uncertainty}, predicting matches/tournaments \citep{cattelan2013dynamic,tsokos2019modeling,macri2024alternative}, testing the efficiency of betting markets \citep{mchale2011bradley,lyocsa2018bet,ramirez2023betting}, and refinement of large language models based on human evaluations \citep{christiano2017deep,ouyang2022training,zhu2023principled}.  

While the models mentioned above have made significant contributions to the field, they rely on the assumption of stochastic transitivity, which implies a strict ranking among players/teams/items. However, this assumption may be unrealistic, particularly in settings involving multiple skills or strategies, where intransitivity naturally arises. 
Despite its practical importance, research on models that allow intransitivity remains limited. Some notable exceptions include the work of \cite{chen_joachims-2016-proceedings} and \cite{spearing_etal-2023-JCGS}, which extend the Bradley-Terry model by introducing additional parameters to describe intransitivity alongside parameters specifying absolute strengths based on Bradley-Terry probabilities.  \cite{spearing_etal-2023-JCGS} propose a Markov chain Monte Carlo algorithm for parameter estimation under a full Bayesian framework. However, their Bayesian procedure is computationally intensive and impractical for high-dimensional settings involving many players or a relatively high latent dimension. \cite{chen_joachims-2016-proceedings} treat the parameters as fixed quantities and estimate them by optimizing a regularized objective function. However, their objective function is non-convex, and their model is highly over-parameterized. Consequently, their optimization is still computationally intensive and does not have a convergence guarantee. 
Moreover, no theoretical results are established in either work for their estimator. 

Motivated by these challenges, we propose a general framework for modeling intransitive pairwise comparisons, assuming an approximately low-rank structure for the pairwise comparison probability matrix. We propose an estimator for the probabilities, which can be efficiently solved by a convex optimization program. This estimator
is shown to be optimal in the minimax sense, accommodating sparse data—a common issue when the number of players diverges. To our knowledge, this is the first framework to address intransitive comparisons with rigorous error analysis. The models presented in \cite{chen_joachims-2016-proceedings} and  \cite{spearing_etal-2023-JCGS}, which assume a low-rank structure, can be seen as a special case of our framework. Furthermore, our method and computational algorithms scale efficiently to high-dimensional settings, making them suitable for applications with many players/teams/items. Empirical results on real-world datasets, including the e-sport \textit{StarCraft II} and professional tennis,  demonstrate the practical usefulness of our method, showing superior performance in intransitive settings and robust performance when transitivity largely holds. 

Pairwise comparison data has been extensively studied in the statistics and machine learning literature, with numerous models and methods developed. We refer readers to \cite{cattelan2012models} for a practical overview of the field. Theoretical properties of the BT model were first established in \cite{simons1999asymptotics}. These results were later extended to likelihood-based and spectral estimators, as well as other parametric extensions, with various losses and sparsity levels \citep{yan2012sparse,shah_etal-2016-JoMR,Sahand_etal-2017-OR,chen2019spectral,han2020asymptotic,chen2022partial}. More recently, \cite{Han_etal-2023-JASA} propose a general framework covering most parametric models satisfying strong stochastic transitivity, establishing uniform consistency results under sparse and heterogeneous settings.

Our development is also closely related to the literature on generalized low-rank and approximate low-rank models \citep{cai2013max,davenport_etal-2014,cai2016matrix,Chen_Li_Zhang-2020-JASA,chen_Li-2022-Biometrika,chen_Li-2024-JoMR,lee2026latent}. While our asymptotic results and error bounds build on techniques from these works, the parameter matrix in the current work differs in that it has a skew-symmetric structure. 
This structure, which arises naturally from pairwise comparison data, yields dependent entries and distinguishes our setting from typical low-rank models. To address this, a tailored analysis is performed to establish rigorous theoretical results.

The rest of the paper is organized as follows. Section \ref{sect: Generalised approximate low-rank model for pairwise comparison data} describes the setting, introduces the general approximate low-rank model, and proposes our estimator. Section \ref{sect: Theoretical results} establishes the theoretical properties of the proposed estimator, including results on convergence and optimality. In Section \ref{sect: computation}, we provide an algorithm for solving the optimization problem of the proposed estimator. Section \ref{sect: simulation}  verifies the theoretical findings and compares the proposed model with the BT model using simulations. Section \ref{sect: real data} applies the proposed method to two real datasets to explore the presence of intransitivity in sports and e-sports. Finally, we conclude with discussions in Section \ref{sect: Discussions}. Additional theoretical results are presented in Appendix~\ref{app: additional theoretical results}, detailed proofs of the main results are given in 
Appendix~\ref{app: Proof of Theoretical Results}, and additional simulation results are reported in Appendix~\ref{app: Additional Simulation Results}.

\section{Generalized Approximate Low-rank Model for Pairwise Comparisons}\label{sect: Generalised approximate low-rank model for pairwise comparison data}
\subsection{Setting and Proposed Model}
We consider a scenario with $n$ subjects, such as players in a sports tournament. Let $n_{ij}$ denote the total number of comparisons observed between subjects $i$ and $j$, where $(n_{ij})_{n \times n}$ is a symmetric matrix. Let $y_{ij}$ denote the observed counts where subject $i$ beats subject $j$. Assuming no draws, we have $y_{ij} = n_{ji} - y_{ji}$ for $i, j \in \{1, \dots, n\}$. 

Given the total comparisons $n_{ij}$, we model the observed counts $y_{ij}$ using a Binomial distribution: $y_{ij} \sim \text{Binomial}(n_{ij}, \pi_{ij}),$ where $\pi_{ij}$ denotes the probability that subject $i$ beats subject $j$. A fundamental property of the probabilities is that $\pi_{ij} = 1 - \pi_{ji}$ for all $i,j \in \{1, \dots, n\}$. This implies that the matrix $\Pi = (\pi_{ij})_{n\times n }$ is fully determined by its upper triangular part. Using the logistic link function $g(x) = (1 + \exp(-x))^{-1}$, we express the probabilities as $\pi_{ij} = g(m_{ij})$, where $M = (m_{ij})$ is a skew-symmetric matrix satisfying $M = -M^{\top}$. As a result, estimating the probabilities $\Pi$ reduces to the problem of estimating $M$. 

We say the model is stochastic transitive if there exists an unobserved global ranking among all the players, denoted by $i_1 \succ i_2 \succ \cdots \succ i_n$, such that the pairwise comparison probabilities for the adjacent pairs satisfy $\pi_{i_1 i_2}$, $\pi_{i_2 i_3}$, ..., $\pi_{i_{n-1} i_n} \geq 0.5$. In addition, 
$\pi_{ik} \geq \pi_{ij}$ whenever $j \succ k$, for all $i \neq j,k$. In other words, for two players, $j$ and $k$, any player is more likely to win $k$ than $j$ if player $j$ ranks higher than $k$. If stochastic transitivity does not hold, then we say a model is stochastic intransitive. For instance, stochastic intransitivity arises when there exists a triplet $(i,j,k)$, such that 
$\pi_{ik} \geq \pi_{ij}$ and $\pi_{jk} < 0.5$.

Most traditional models for pairwise comparison assume stochastic transitivity. For example, the BT model assumes $m_{ij} = u_i -u_j$, in which case, the global ranking of the players is implied by the ordering of $u_i$, $i =1, ..., n$. However, stochastic intransitivity naturally occurs in real-world competition data involving multiple strategies or skills. For example, in the professional competitions of the e-sport \textit{StarCraft II}, players can choose from a variety of combat units with differing attributes (e.g., building cost, attack range, toughness) during the game, leading to strategic decisions that can result in intransitivity. In fact, for the best predictive model that we learned for the \textit{StarCraft II} data, more than 70\% of the $(i,j,k)$ 
triplets are estimated to violate the stochastic transitivity assumption, i.e.,  $\pi_{ik} \geq \pi_{ij}$ and $\pi_{jk} < 0.5$; see Section~\ref{sect: real data} for the details.

From the modeling perspective, stochastic transitivity is achieved by imposing strong monotonicity constraints on the parameter matrix $M$. To allow for stochastic intransitivity, we need to relax such constraints. 
Given $Y = (y_{ij})_{n \times n}$, the log-likelihood is 
\begin{align*} 
    \LL(M)&=  \sum_{i=1}^{n} \sum_{j=1}^{n} y_{ij}\log(  g( m_{ij} ))\\
            & = \sum_{i=1}^{n} \sum_{j > i} \left( y_{ij}\log(  g( m_{ij} )) + (n_{ij} - y_{ij})\log( 1- g(m_{ij})) \right).
\end{align*}
To prevent overfitting while accommodating stochastic intransitivity, we impose a constraint on $M$ to reduce the size of the parameter space. Specifically, we assume that $M$ has an approximately low-rank structure enforced through a nuclear norm constraint:
\begin{align}\label{eq: model}
  \|M\|_* \leq C_n n,
\end{align}
where $\|\cdot\|_*$ denotes the nuclear norm, and $C_n > 0$ is a constant that may vary with $n$.
The estimator is defined as:
\begin{align}\label{eq: nuclear norm estimator}
    \hat{M} = \argmax_{M} \LL(M) \text{ subject to } \|M\|_{*} \leq C_n n, M = - M^{\top}.
\end{align}
It is easy to see that the optimization in \eqref{eq: nuclear norm estimator} is convex; see Section \ref{sect: computation} for its computation. 

\subsection{Comparison with Related Work}
We compare the proposed model with existing parametric models in the literature. \cite{Han_etal-2023-JASA} introduce a general framework for analyzing pairwise comparison data under the assumption of stochastic transitivity. In the current context, their model aligns with those proposed by \cite{shah2016stochastically} and \cite{heckel2019active}, which are expressed as  
\[
\pi_{ij} = \Phi(u_i - u_j), \mbox{~and~} \pi_{ji} = 1 - \Phi(u_i - u_j).
\]  
Here, \( \Phi(\cdot) \) is any valid symmetric cumulative distribution function specified by the user, and \( \uu = (u_1, \dots, u_n)^{\top} \) is a latent score vector representing the strengths of the teams. This framework reduces to the Bradley-Terry (BT) model when \( \Phi(\cdot) = g(\cdot) \), the logistic function, and to the Thurstone model when \( \Phi(\cdot) \) is the cumulative distribution function of the standard normal distribution. Other models can be incorporated by specifying different forms of \( \Phi(\cdot) \).  The latent score \( \uu \) is treated as a fixed parameter to be estimated, enabling the framework to handle a large number of players effectively.
This parametric form, however, enforces a rank-2 structure on the parameter matrix, given by  
\[
\Pi = \Phi(\uu \mathbf{1_n}^{\top} - \mathbf{1_n} \uu^{\top}),
\]  
where \( \mathbf{1_n} \) is an \( n \)-dimensional vector of ones.   

Several attempts have been made in the literature to generalize this parametric form, allowing the rank of the underlying parameter matrix to exceed two and accommodate stochastic intransitivity.  We should note that, since $M$ is a skew-symmetric matrix, its rank must be even \citep[e.g.,][]{horn2013matrix}. 
For instance, \cite{chen_joachims-2016-proceedings} proposed a blade-chest-inner model, which can expressed as
\begin{align*}
    \Pi = g( AB^{\top} - B A^{\top}  ),
\end{align*}
where $A$ and $B$ are $n \times K$ matrices. This model allows for a general rank-$2k$  parameter matrix, with the parameters in the frequentist sense. 
Similar to the parametrization in \cite{chen_joachims-2016-proceedings}, \cite{spearing_etal-2023-JCGS} propose a Bayesian model for pairwise comparison under stochastic intransitivity and further develop a Markov chain Monte Carlo algorithm for its computation. 
Both methods lack theoretical guarantees, such as convergence results or error bounds.

Our proposed method relaxes the requirement for an exact low-rank representation by only requiring an approximate low-rank structure specified by the nuclear norm. This offers a broad parameter space that covers the models proposed in \cite{chen_joachims-2016-proceedings} and \cite{spearing_etal-2023-JCGS}, thereby providing improved robustness against model misspecification. This flexibility is particularly important for real-world applications, where comparison probabilities may be influenced by numerous weak factors alongside a few dominant ones, so that the parameter matrix need not exhibit an exact low-rank structure. In particular, if \( \text{rank}(M) = 2k \) for some positive integer \( k \), it follows that  
\[
\|M\|_{*} \leq \sqrt{2k}\|M\|_{F} \leq C_n n,
\]  
where \( \|\cdot\|_{F} \) denotes the Frobenius norm, and \( C_n \) is a constant depending on the magnitude of the entries of \( M \) and its rank \( 2k \). The subscript $n$ in $C_n$ indicates that both the magnitude of the entries of $M$ and its rank are allowed to grow with $n$. Moreover, the proposed model imposes no distributional assumptions on the parameter matrix \( M \), \rev{and the estimator can be obtained by solving a convex optimization program,} making it more scalable \rev{and computationally tractable} for handling a large number of players. \rev{Furthermore, as illustrated in Section \ref{sect: real data}, practical implementation requires selecting tuning parameters for the nuclear norm constraint. Unlike the exact low-rank formulation, which requires selecting the integer-valued exact rank, the tuning parameter can be selected over a flexible grid, allowing improved accuracy through finer tuning when computational resources permit.} Theoretical results, including convergence and error bounds, are presented in Section \ref{sect: Theoretical results}, \rev{with additional results provided in Appendix~\ref{app: additional theoretical results}.} As a remark, our estimation method and theoretical framework can be easily adapted when we replace the current assumption of the logistic form of the link function \( g(\cdot) \) with other functions, such as the standard normal cumulative distribution function used in the Thurstone model. 

\section{Theoretical Results}\label{sect: Theoretical results}
 We establish convergence results and lower bounds for the estimator defined in \eqref{eq: nuclear norm estimator} under settings with different data sparsity levels. For positive sequences $\{a_n\}$ and $\{b_n\}$, we denote $a_n \lesssim b_n$ if there exists a constant $\delta>0$ that $a_n \leq \delta b_n$ for all $n$. \rev{We write $a_n \asymp b_n$ if $a_n \lesssim b_n$ and $b_n \lesssim a_n$.} Let $\KK$ denote the parameter space, defined as 
 \begin{align}\label{eq: defintion of KK}
     \KK = \{M\in \RR^{n \times n}:\|M\|_* \leq C_n n, \quad M = -M^{\top}  \}.
 \end{align}
 We impose the following conditions:

\begin{assumption}\label{assp: 2}
 The true parameter $M^{*} \in \KK$.
\end{assumption}

\begin{assumption}\label{assp: 3}
  For $j = 1,\dots, n$ and  $i >j$, the variables $n_{ij} $ are independent and follow a Binomial distribution, $n_{ij} \sim \text{Binomial} (T, p_{ij,n})$, where $T$ is a fixed integer representing the maximum possible number of comparisons between subjects, and $p_{ij,n} = p_{ji,n}$ is the \rev{comparison rate}, which may vary across different pairs $(i, j)$. Let \( 0 \leq p_n \leq q_n \leq 1 \) denote the minimum and maximum comparison rates, respectively, such that \( p_{ij,n} \in [p_n, q_n] \) for all \( i \neq j \in \{1, \dots, n\} \).  We assume that $p_n \asymp q_n$ and $p_n \gtrsim \log(n)/n$. 
\end{assumption}
Assumption \ref{assp: 2} ensures that the true parameter exhibits an approximately low-rank structure specified by our model. Assumption \ref{assp: 3} deserves more explanations. Under this assumption, the sparsity level of the data is characterized by the rate at which the \rev{comparison rates} $p_{ij,n}$ converge to $0$ as $n$ grows. The condition \( p_n \gtrsim \log(n)/n \) sets a lower bound on the sparsity level, which is the best possible threshold for pairwise comparison problems. Below this bound, the comparison graph becomes disconnected with high probability \citep{erd6s1960evolution,Han_etal-2023-JASA}. The condition \( p_n \asymp q_n \) imposes homogeneity on \( p_{ij,n} \), a common assumption in the literature \citep{simons1999asymptotics,chen2019spectral,han2020asymptotic}. \rev{This assumption means that all pairs of items are observed with roughly equal probability.} 
The following theorem establishes the convergence rate of the proposed estimator.
\begin{theorem}\label{thm: nuclear norm convergence} 
     Under Assumptions \ref{assp: 2} and \ref{assp: 3}, let $\hat{\Pi} = (\hat{\pi}_{ij})_{n\times n}$, where 
    $\hat{\pi}_{ij} = g(\hat{m}_{ij}).$ Further let $\Pi^* = g(M^*)$. Then, with probability at least $1 - \kappa_1/n$, 
    \begin{align*}
        \frac{1}{n^2-n}\|\hat{\Pi} - \Pi^{*}\|^2_{F} \leq  \kappa_2 C_n \sqrt{\frac{1}{p_n n}},
    \end{align*}
    where $\kappa_1$ and $\kappa_2$ are constants that do not depend on n. 
\end{theorem}
 \rev{As shown in Theorem \ref{thm: nuclear norm convergence}, the convergence rate of the mean squared Frobenius error between the estimated and true pairwise comparison probabilities depends on the sample size, sampling density, and model complexity. Specifically, the error decreases as the number of subjects $n$ increases and as the sampling becomes denser (i.e., as $p_n$ increases). In contrast, the upper bound increases when the underlying parameter matrix has a higher rank or a more complex structure (as reflected by a larger value of $C_n$). } 
 
 \rev{The following theorem addresses the optimality of Theorem \ref{thm: nuclear norm convergence} by providing a lower bound on the achievable estimation error, confirming that the rate in Theorem 1 cannot be improved in general.}  
\begin{theorem}\label{thm: lower bound}
    Suppose $ 12 \leq C_n^2 \leq \rev{\min\{1, \kappa_3^2\}n}$, where $\kappa_3$ is an absolute constant specified in \rev{Theorem \ref{thm: lower bound nij given} in Appendix \ref{app: Proof of thm: lower bound}}. Consider any algorithm which, for any $M\in \KK$, takes as input $Y$ and returns $\hat{M}$. Then there exists $M \in \KK$ such that with probability at least 3/8, $\Pi = g(M)$ and $\hat{\Pi} = g(\hat{M})$, satisfy
    \begin{align}\label{eq: Pi lower bound}
        \frac{1}{n^2-n}\|\Pi - \hat{\Pi}\|_{F}^2 \geq \min \left\{\kappa_{4}, \kappa_5 C_n\sqrt{\frac{1}{np_n} } \right\}
    \end{align}
    for all $n > N$. Here $\kappa_4, \kappa_5 >0$ and $N$ are absolute constants.
\end{theorem}
A few technical assumptions are imposed in this theorem. The condition \rev{\( C_n^2 \leq \min\{1, \kappa_3^2\}n \)} is mild and naturally holds for sufficiently large \( n \), provided that the rank of \( M \) does not grow at the same rate as \( n \). We also require $C_n^2 \geq 12$ to avoid the parameter space being too small for packing set construction\rev{, which is needed when we use Fano's inequality to establish the lower bound.}

Since the rates in Theorems \ref{thm: nuclear norm convergence} and \ref{thm: lower bound} match up to a multiplicative constant, the optimality of the proposed estimator is established. 

\begin{remark}
\rev{
Some assumptions can be relaxed in the above theoretical analysis. Specifically, a version of Theorem~\ref{thm: nuclear norm convergence} continues to hold with an updated convergence rate when the assumption $p_n \asymp q_n$ is dropped. Moreover, if $T$ is not fixed but diverges, a version of Theorem~\ref{thm: nuclear norm convergence} holds with a faster convergence rate; see Remark~\ref{rmk: app relax assmp} in Appendix \ref{app: Proof of Theorem thm: nuclear norm convergence} for details.}
\end{remark}

\begin{remark}\label{rmk: l2 and max norm convergence}
\rev{
In this section, we establish the convergence of the matrix of pairwise comparison probabilities $\Pi^*$ under the Frobenius norm. It is also possible to derive convergence results for the underlying parameter matrix $M^*$ under additional assumptions, specifically by requiring that $\|M^*\|_{\infty}$ is bounded by a positive constant, as demonstrated in the proof of Lemma \ref{lm: F consistent} in Appendix \ref{app: proves of thm: 2 to infty}. 

Furthermore, under an exact low-rank assumption on $M^*$, convergence of $M^*$ in the element-wise maximum norm can be established by making use of the refinement procedure proposed in \cite{chen_Li-2024-JoMR}. This implies the  entry-wise convergence for the comparison probabilities. Such results also yield control of the two-to-infinity norm of the low-rank component, which may be particularly relevant when studying player-specific skill parameters. 
  
A detailed discussion of the exact assumptions, along with the corresponding estimation and refinement procedures, is provided in Appendix \ref{app: Assumption and refinement}.}
\end{remark}

\begin{remark}
    \rev{The proposed model does not impose stochastic transitivity and therefore a global ranking is not well-defined by the true comparison probability matrix. Nevertheless, for any subset of subjects $\CS \subseteq [n]$, one may define a relative score measuring the average strength of subject $i$ within $\CS$, which reduces to the score in \cite{shah_wainwright-2018-JoMR} when $\CS = [n]$. 
Under a suitable separation condition, we establish consistent recovery of the top-$k$ set based on these scores. 
The precise formulation and theoretical results are presented in Appendix \ref{app: Recovery of Top-k Items}.}
\end{remark}
\section{Computation}\label{sect: computation}
To solve the nuclear-norm constrained optimization problem \eqref{eq: nuclear norm estimator}, we apply the nonmonotone spectral-projected gradient algorithm proposed by \cite{birgin2000nonmonotone}. \rev{Since the feasible set of $M$ is closed and convex, this algorithm} guarantees convergence to a \rev{constrained stationary point. Moreover, because $\LL(M)$ is concave, any such stationary point is a global maximizer over the feasible set.}

\rev{
This section outlines the computational procedure for the proposed model. 
Section~\ref{subsect: Reformulating the optimization problem} reformulates the optimization problem \eqref{eq: nuclear norm estimator} in vectorized form. 
Section~\ref{subsect: Projection onto the nuclear-norm ball} presents the projection algorithm used to enforce the nuclear-norm constraint. 
Section~\ref{subsect: Spectral projected gradient algorithm} describes the spectral projected gradient method for solving the constrained problem, and 
Section~\ref{subsect: Estimation procedure} presents the complete estimation procedure.
}
\subsection{\rev{Reformulating the optimization problem}}\label{subsect: Reformulating the optimization problem}
 \rev{We begin by reformulating the matrix optimization problem into a vectorized form that is more convenient for numerical computation.} Let $\text{Skew}_n$ denote the space of $n \times n$ skew-symmetric matrices. Let $\VV$ be the bijective linear mapping that vectorizes the upper-triangular part of any matrix in $\text{Skew}_n$ into $\RR^{0.5n(n-1)}$. For any \( \mm \in \mathbb{R}^{0.5n(n-1)} \), define \( f(\mm) = \LL(\VV^{-1}(\mm)) \). Then, solving \eqref{eq: nuclear norm estimator} is equivalent to solving the constrained optimization problem:
\begin{align}\label{eq: vec objective}
    \hat{\mm} = \argmax_{\mm \in \mathbb{R}^{0.5n(n-1)}} f(\mm) \quad \text{subject to } \|\VV^{-1}(\mm)\|_{*} \leq \tau,
\end{align}
where \( \tau = C_n n \) if \( C_n \) is known. We will later discuss an algorithm for selecting \( \tau \) in practical situations where \( C_n \) is unknown. 
\subsection{\rev{Projection onto the nuclear-norm ball}}\label{subsect: Projection onto the nuclear-norm ball}
\rev{A key step in solving \eqref{eq: vec objective} is to project each iterate onto the nuclear-norm ball in order to enforce the constraint. 
We define the orthogonal projection operator $P_{\tau}(\cdot)$ as}
\begin{align*}
    P_\tau(\mm) = \argmin_{\xx \in \mathbb{R}^{0.5n(n-1)} } \|\xx - \mm \|_2 \text{ subject to } \|\VV^{-1}(\xx)\|_{*} \leq \tau. 
\end{align*}
It is well known that the projection is equivalent to singular value soft-thresholding. \rev{Specifically, we determine a threshold and update each singular value by either subtracting this threshold from it or setting it to zero when the threshold exceeds the singular value, so that the resulting matrix satisfies the nuclear-norm constraint.} Let $0_{n \times n }$ denote a $n \times n $ zero matrix, and $\max\{\cdot, \cdot\}$ be applied entry-wise for matrix inputs. The detailed procedure is presented in Algorithm \ref{alg: projection}, \rev{where we first apply the inverse mapping $\VV^{-1}(\cdot)$ to recover the corresponding skew-symmetric matrix, perform singular value soft-thresholding to enforce the nuclear-norm constraint, and then apply $\VV(\cdot)$ to return the vectorized projected output.}

\begin{algorithm}
    \caption{Projection onto the nuclear-norm ball} \label{alg: projection}
    \textbf{Input:} Parameter vector $\mm$ and nuclear norm constraint parameter $\tau$.
    \begin{algorithmic}
        \State Compute $M = \VV^{-1}(\mm)$. 
        \State Perform singular value decomposition and obtain $M = U \Sigma V^{\top} $, where $U$ and $V$ are $n \times n$ orthonormal matrices and $\Sigma = \text{diag}(\sigma_1,\sigma_1, , \dots,\sigma_{n/2}, \sigma_{n/2})$ if $n$ is even and $\Sigma = \text{diag}(\sigma_1,\sigma_1, , \dots,\sigma_{\lfloor n/2 \rfloor}, \sigma_{\lfloor n/2 \rfloor},0)$ otherwise.
        \State Compute $\lambda$, the smallest value for which $\sum_{i=1}^{\lfloor n/2\rfloor} 2\max\{ \sigma_i - \lambda,0\} \leq \tau$. 
        \State Compute projected matrix $P_{\tau}(M)  = U \max\{\Sigma - \lambda I_{n},0_{n \times n}\} V^{\top}$  
    \end{algorithmic}
    \textbf{Output:} Projection outcome $P_{\tau}(\mm) = \VV( P_{\tau}(M))$.
\end{algorithm}

In the last step, the projection outcome is defined as $P_{\tau}(\mm) = \VV( P_{\tau}(M))$, which is only valid provided that $P_{\tau}(M)$ is a skew-symmetric matrix. The following proposition ensures that this is always the case: 
\begin{proposition}\label{prop: skew preserving}
For any matrix $M \in \text{Skew}_n$, the projection operator satisfies $P_{\tau}(M) \in \text{Skew}_n$. 
\end{proposition}

\subsection{\rev{Spectral projected gradient algorithm}}\label{subsect: Spectral projected gradient algorithm}
We now introduce the spectral projected line search method. \rev{This approach combines gradient-based updates with projection onto the nuclear-norm ball, ensuring that each iterate remains} within the feasible set defined by the nuclear norm constraint. The procedure is outlined in Algorithm \ref{alg: line search}.

\begin{algorithm}[h]
    \caption{Spectral projected line search} \label{alg: line search}
    \textbf{Input:} Parameter vector from last iteration $\mm^{(l-1)}$, Matrix of comparison outcomes $Y$, nuclear norm constraint parameter $\tau$ and the spectral-step length $\gamma_{l-1}$
    \begin{algorithmic}
        \State Compute gradient: $\gg^{(l-1)} = \nabla f(\mm^{(l-1)})$.
        \State Compute search direction: $\dd^{(l-1)} = P_{\tau}( \mm^{(l-1)} - \gamma_{l-1}\gg^{(l-1)} ) - \mm^{(l-1)}$
        \State Perform line search along the linear trajectory: $\mm(\alpha) = \mm^{(l-1)} + \alpha \dd^{(l-1)}$.  
        \If{Convergence is reached}
        \State Set $\mm^{(l)}$ as the result from the line search.
        \Else 
        \State Perform line search along the alternative trajectory:  
        $$\mm^{\text{curve}}(\alpha) = P_{\tau}(\mm^{(l-1)} - \alpha \gamma_{l-1} \gg^{(l-1)}). $$ 
        \State Set $\mm^{(l)}$ as the result from the line search.
        \EndIf 
    \end{algorithmic}
    \textbf{Output:} Updated parameter vector $\mm^{(l)}$. 
\end{algorithm}

\rev{At each iteration, after computing the gradient at the current iterate,} the method employs two types of line searches. The first type performs a projection once and searches along a linear trajectory \( \mm(\alpha) \). This approach is computationally efficient since the primary computational cost lies in the projection operation. If the linear search fails to converge, the algorithm switches to a curvilinear trajectory \( \mm^{\text{curve}}(\alpha) \), which requires projecting at each step. The Spectral-step length $\gamma_{l-1}$ is decided using the method from \cite{barzilai1988two} in each iteration.

\subsection{\rev{Estimation procedure}}\label{subsect: Estimation procedure}
\rev{We now summarize the full estimation procedure by integrating the reformulation, projection, and spectral projected gradient steps described above.} The full procedure is detailed in Algorithm \ref{alg:estimation_algorithm}. The convergence criterion checks whether the optimality condition \( P_{\tau}(\mm^{(l)} - \nabla f(\mm^{(l)})) = \mm^{(l)} \) is approximately satisfied. Parts of the code are adapted from the SPGL1 package, originally implemented in \textsc{Matlab} \citep{van2008probing,davenport_etal-2014}. The proposed estimator is implemented in \textsc{R}, and the code \rev{for the implementation, as well as the simulation and real data analyses,} is available at \url{https://github.com/Arthurlee51/PCWST}. 

\begin{algorithm}[h]
\caption{Estimation Algorithm}\label{alg:estimation_algorithm}
\textbf{Input:} Matrix of comparison outcomes $Y$, nuclear norm constraint parameter $\tau$.
\begin{algorithmic}
\State \textbf{Initialization:} Set $l=0$, $\mm^{(0)} = \mathbf{0}_{0.5n(n-1)}$, the zero vector and set the spectral step-length $\gamma_{0} = 1$.
\While{$l = 0$ \textbf{or} convergence criterion is not satisfied }
    \State Update $l \gets l + 1$.
    \State Update $\mm^{(l)}$ via line search using Algorithm \ref{alg: line search} with inputs $\mm^{(l-1)}$, $Y$, $\tau$ and $\gamma_{l-1}$.
    \State Update $\gamma_{l}$ as proposed by \cite{barzilai1988two}.
 \EndWhile
 \end{algorithmic}
\textbf{Output:} Estimated parameter matrix $\hat{M} = \VV^{-1} (\mm^{(l)})$. 
\end{algorithm}

\section{Simulation Results}\label{sect: simulation}
\rev{This section presents simulation studies to validate the theoretical results in Section~\ref{sect: Theoretical results} and assess the numerical performance of the proposed estimator. Specifically, we evaluate the proposed model in terms of the estimation error of the pairwise comparison probability matrix and its predictive likelihood on independently generated test data. We compare its performance with that of the BT model across varying sample sizes, matrix complexities, and sparsity levels.}

We consider three distinct scenarios characterized by varying levels of sparsity. Specifically, we define $p_n$ as $n^{-1} \log(n)$, $n^{-1/2}$, and $1/4$, corresponding to sparse, less sparse, and dense data, respectively. The parameter $q_n$ is given by $4p_n$. Each $p_{ij,n}$ is then generated from a uniform distribution with range $[p_n,q_n]$. 

The parameter matrix $M$ is constructed as $\Theta J \Theta^{\top}$, where $\Theta$ is an $n \times 2k$ matrix, and $J$ is a $2k \times 2k$ block diagonal matrix of the form
\begin{align*}
    J = \begin{pmatrix}
    0 & n & 0 & \dots & 0 \\
    -n & 0 & 0 & \dots & 0 \\
    0 & 0 & \ddots & \dots & 0 \\
    \vdots & \vdots & \vdots & 0& n \\
    0 & 0 & 0 & -n & 0
\end{pmatrix}.
\end{align*}
The matrix $\Theta$ is orthonormal, obtained via QR decomposition of a random matrix $Z\in \mathbb{R}^{n \times 2k}$, where each entry of $Z$ is independently sampled from a standard normal distribution $N(0, 1)$. It can be verified that $\|M\|_{*} = 2kn$. 

We conduct $50$ simulations for $n = 500, 1000, 1500$, and $2000$, with $k$ ranging from 1 to 10. Recall that the rank of $M$ is $2k$. 
Additionally, the maximum number of comparisons, $T$, is fixed at 5, across all settings. We set $C_n = 2k$. The loss is computed as
\begin{align}\label{eq: loss}
\text{Loss} = (n^2 - n)^{-1}\| \hat{\Pi} - \Pi^* \|_F^2,
\end{align}
and the average loss across 50 simulations is reported for each model in Figure \ref{fig:comparison}, considering different values of $n$, $k$, and sparsity levels. 

The results in Figure \ref{fig:comparison} show that the mean loss of the proposed estimator decreases as $n$ increases. Moreover, the mean loss is significantly lower as the data become denser, corresponding to an increase in $p_n$. These observations are consistent with the results from Theorem~\ref{thm: nuclear norm convergence}. 

Notably, the proposed and BT models incur higher losses as the rank parameter $k$ increases, which is expected due to increasing complexity. However, the proposed model consistently outperforms the BT model across all settings, \rev{as reflected by the lower mean losses and the fact that its error bars lie entirely below those of the BT model.}  Furthermore, while the BT model’s performance remains relatively unchanged as $n$ increases, the proposed method continues to improve, showcasing its effectiveness in handling large datasets and capturing complex structures that stochastic transitivity assumptions cannot address.

\rev{In addition to estimation error, we evaluate predictive performance using the log-likelihood computed on an independently generated test set. The results further confirm the superiority of the proposed method over the BT model, as evidenced by consistently higher average predictive likelihood values across all settings. The detailed results, together with the computational time of both methods, are presented in Appendix~\ref{app: Additional Simulation Results}.}

\begin{figure}[t]
    \centering
    \includegraphics[width=\linewidth]{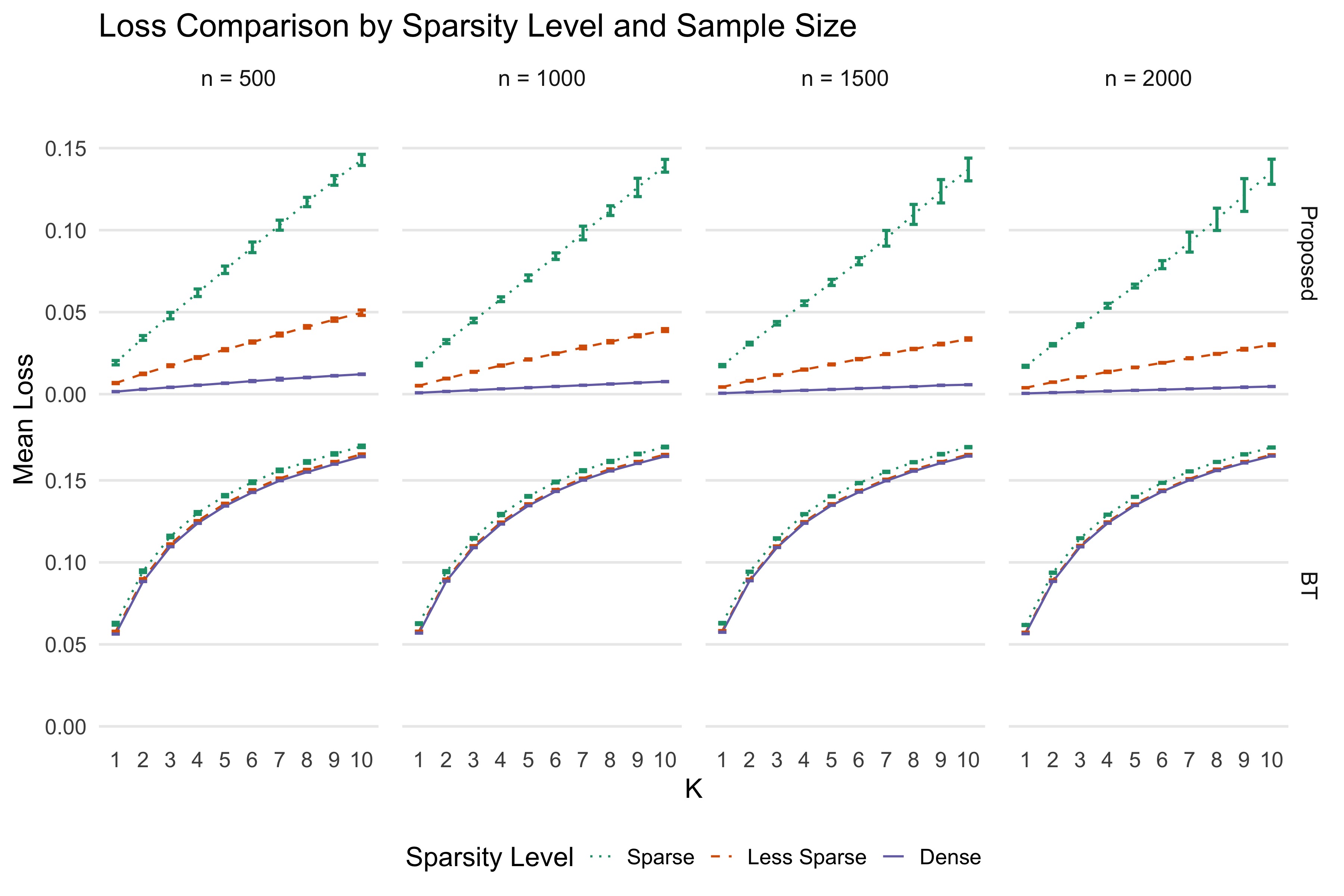}
    \caption{Comparison of loss between the proposed method and the Bradley-Terry (BT) model across different sparsity levels (sparse, less sparse, dense). \rev{Each row corresponds to a method (top: proposed; bottom: BT), and each column corresponds to a sample size ($n = 500, 1000, 1500, 2000$).} The x-axis represents the rank parameter \( k \), while the y-axis shows the mean loss, computed as the average of the losses defined in \eqref{eq: loss}. \rev{Error bars indicate $\pm 2$ standard deviations across replicates.}}
    \label{fig:comparison}
\end{figure}

\section{Real Data Examples} \label{sect: real data}
In this section, we compare our model's performance with the \rev{classical} BT model using two real datasets. \rev{Section \ref{subsect: datasets} provides a brief description of the two datasets used in the analysis, namely the StarCraft II and Tennis datasets.}  Section \ref{subsect: Data Preparation and Parameter Tuning} outlines the data preparation process and describes how the nuclear norm constraint parameter $\tau = C_n n$ is decided. Section \ref{subsect: Evluation Criteria} introduces the evaluation metrics used to compare the models. \rev{Finally, Section \ref{subsect: Empirical Results} reports the empirical results for both the StarCraft II and Tennis datasets.}
\subsection{\rev{Description of Datasets}}\label{subsect: datasets}
\subsubsection{\textit{StarCraft II} Data} \label{subsubsect: StarCraft II}
\textit{StarCraft II} is a military science fiction real-time strategy game developed and published by Blizzard Entertainment. The dataset comprises match results of professional \textit{StarCraft II} players sourced from the website \url{aligulac.com}, covering the period from 2010 to 2016. The matches follow the most common competitive format, where two players face off against each other, and each game results in either a win or loss, with no possibility of a draw.

We specifically focus on matches played using the \textit{StarCraft II: Heart of the Swarm} expansion, as different versions of the game are often treated as distinct games \citep{chen_joachims-2016-proceedings}. The training set includes $1,958$ players, with $1.9\%$ of all player pairs competing against each other at least once. The maximum number of matches between any pair of players is 30. The dataset is available at \url{https://www.kaggle.com/datasets/alimbekovkz/starcraft-ii-matches-history}.
\subsubsection{Tennis Data} \label{subsubsect: Tennis}
We analyze the tennis dataset to evaluate the performance of our model in professional sports. The dataset contains the results of all men's matches organized by the Association of Tennis Professionals (ATP) from 2000 to 2018. It includes matches from major tournaments such as the Grand Slams, the ATP World Tour Masters 1000, and other professional tennis series held during this period.

The training set consists of $723$ players, with $6.4\%$ of all player pairs having competed against each other at least once. The maximum number of matches between any pair of players is 23. The data is collected from \url{http://www.tennis-data.co.uk}.

\subsection{Data Preparation and Parameter Tuning} \label{subsect: Data Preparation and Parameter Tuning}
 The raw data consists of individual match records, with each comparison recorded as a separate entry. We reserve 30\% of the match records for testing, while the remaining 70\% is divided into 50\% for training and 20\% for validation.

The comparison data matrix is first constructed for the training set, with players absent from the training set removed from the validation set. The validation set is used to tune the nuclear constraint parameter $C_n$, as described in the sequel. After tuning, the training and validation sets are combined (including previously excluded entries), and the comparison data matrix is reconstructed from the combined dataset. 

The test set is then evaluated against this combined dataset, excluding entries for players not present in the combined dataset. Although the proposed model can handle players who never lose or win any game, we still remove them in the training and combined dataset to ensure stabler results and a fair comparison with the BT model, as this is a common practice. 

The nuclear norm of the parameter matrix \( M \) is unknown and is tuned on the training and validation sets using log-likelihood as the loss function.  The nuclear constraint parameter \(\tau = C_n n\) is determined by selecting \( C_n \) from 20 grid points, corresponding to powers of 10 evenly spaced between \(-1\) and \(1\). This results in \( C_n = 10^{0.47} = 2.98 \) for the \textit{StarCraft~II} dataset and \( C_n = 10^{-0.36} = 0.43 \) for the tennis dataset.

\subsection{Evaluation Criteria}\label{subsect: Evluation Criteria}
Let \( Y^{(\txt)} = (y_{ij}^{(\txt)})_{n \times n} \) denote the observed comparison results from the test set. Given the estimated comparison probabilities \( \hat{\Pi} = (\hat{\pi}_{ij})_{n \times n} \), we evaluate the performance of the estimates using two criteria. The first criterion is the log-likelihood, given by
\begin{align*}
    L(Y^{(test)} \mid \hat{\Pi}) =   \sum_{i=1}^{n} \sum_{j > i} \left( y_{ij}^{(\txt)} \log( \hat{\pi}_{ij} ) + y_{ji}^{(\txt)} \log( 1- \hat{\pi}_{ij} ) \right),  
\end{align*}
where a higher log-likelihood indicates a stronger agreement between the predicted probabilities and the observed results. The second criterion is the test accuracy, given by 
\begin{align*}
    A(Y^{(\txt)} \mid \hat{\Pi} ) = \frac{1}{ \sum_{i=1}^{n} \sum_{j=1}^{n} y_{ij}^{(\txt)} } \sum_{i=1}^{n} \sum_{j > i} \bigg( &  y_{ij}^{(\txt)}I(\hat{\pi}_{ij} \geq 0.5)+  y_{ji}^{(\txt)}I(\hat{\pi}_{ji} > 0.5)   \bigg).
\end{align*}
It measures the proportion of the comparison results correctly predicted, with higher values indicating better predictive performance. The results are presented in Table \ref{tab:real_data}.

\begin{table}[h] 
    \centering
    \begin{tabular}{lrrrr}
        \toprule
                        & \multicolumn{2}{c}{\textit{StarCraft II}} & \multicolumn{2}{c}{Tennis} \\
        \cmidrule(lr){2-3} \cmidrule(lr){4-5}
                        & Proposed & BT         & Proposed & BT      \\
        \midrule
        Log-likelihood            &  $-1,897,946$	&$-2,137,115$   &$-333,076$	    & $-322,483$           \\
        Accuracy               & 0.766             & 0.713         & 0.652             & 0.658                \\
        \bottomrule
    \end{tabular}
    \caption{Comparison of model performance on StarCraft II and ATP datasets. The performance is evaluated using log-likelihood and accuracy for the proposed model and the BT model.}
    \label{tab:real_data}
\end{table}

\subsection{\rev{Empirical Results}}\label{subsect: Empirical Results}
\subsubsection{\textit{StarCraft II} Data} \label{subect: StarCraft II}
As seen in Table~\ref{tab:real_data}, the proposed model achieves a higher log-likelihood of $-1,897,946$ compared to $-2,137,115$ for the BT model. This suggests that our model provides a better fit for the observed test data. The test accuracy of the proposed model is also significantly higher at $0.766$, compared to $0.713$ for the BT model. Among the $1,249,168,756$ distinct triplets in the data, stochastic transitivity is violated in $70\%$ of cases, as indicated by the matrix of estimated probabilities $\hat{\Pi}$ under the proposed model. Specifically, this occurs when there exists an ordering of the three players, denoted as $i$, $j$, and $k$, such that $\hat{\pi}_{ik} \geq \hat{\pi}_{ij}$ and $\hat{\pi}_{jk} < 0.5$.

These results are consistent with previous findings by \cite{chen_joachims-2016-proceedings}, who analyzed a similar dataset over different time frames, suggesting that a strict ranking structure may not be appropriate in e-sports. In particular, intransitivity can naturally arise from game design, such as intransitive relationships among different unit types, which provide players with significant flexibility in choosing units and strategies. Moreover, the strong performance of our method on this dataset confirms its ability to effectively handle sparsity in real-world data, aligning with both simulation and theoretical results.

\subsubsection{Tennis Data} \label{subect: Tennis}
From Table~\ref{tab:real_data}, the BT model achieves a marginally better performance, with a log-likelihood of $-322,483$ compared to $-333,076$ for the proposed model, and a slightly higher test accuracy (0.658 vs 0.652). This advantage may come from the BT model’s smaller parameter space, which is more efficient when the data aligns well with the stochastic transitivity assumption, where the level of intransitivity is minimal or absent. Nevertheless, the performance of the proposed model remains close to that of the BT model, demonstrating its robustness even in settings where transitivity holds. This flexibility is particularly useful when intransitivity is uncertain, as it maintains high accuracy without relying on strict ranking assumptions.

The lack of intransitivity in professional tennis may be due to several factors. Unlike e-sports, tennis offers limited gameplay flexibility, as adjustments to equipment like rackets and shoes have minimal impact compared to the choice of units in \textit{StarCraft II}. Additionally, professional tennis players may be required to be well-rounded as weaknesses are quickly identified and exploited by opponents. In contrast, intransitivity may be more common at lower levels of competition, where skill imbalances are expected to be more significant. For example, a player with a strong serve but weak baseline play may be more likely to defeat one opponent while losing to another with a different style. Investigating intransitivity in lower-tier competitions remains an open question for future research.

\section{Discussions} \label{sect: Discussions}
In this article, we propose a statistical framework for modeling stochastic intransitivity. The framework assumes an approximate low-rank structure in the parameter matrix, expressed through a nuclear norm constraint. Theoretical analysis demonstrates that the proposed estimator achieves optimal convergence rates under a wide range of data sparsity settings. Simulation and empirical analyses confirm that our model is superior to the Bradley-Terry model when the assumption of stochastic transitivity is violated.

Our framework stands apart from the existing literature by imposing an approximate low-rank structure. To our knowledge, all existing methods for pairwise comparison data rely on exact low-rank models, even in the limited works that allow stochastic intransitivity. By accommodating a larger parameter space, our approach offers greater flexibility and applicability to a wider range of datasets. While this may lead to slightly reduced efficiency, our analysis of the tennis dataset demonstrates that the loss of efficiency is small when stochastic transitivity largely holds. 
Therefore, the proposed model may predict pairwise comparison results more accurately 
in many real-world applications. For example, for tournament data, this could lead to more accurate predictions of the champion or the number of rounds each player can play, given historical data and the current tournament schedule. 

The current research may be extended in several directions. Specifically, the current theoretical analysis focuses on the convergence of the loss $\Vert \hat \Pi - \Pi^*\Vert^2/(n^2-n)$, which can be seen as a notion of convergence in an average sense (across entries of the comparison probability matrix). \rev{As shown in Remark~\ref{rmk: l2 and max norm convergence}, under additional structural assumptions, we further establish convergence under the matrix max-norm loss $\Vert \hat \Pi - \Pi^* \Vert_{\infty}$ by leveraging the refinement techniques proposed in \cite{chen_Li-2024-JoMR}. It would be of interest to determine whether such convergence can be obtained under weaker conditions, or whether sharper convergence rates can be achieved through an improved refinement procedure.}  Moreover, it will be useful to further establish the asymptotic normality for each $\hat \pi_{ij} - \pi_{ij}^*$, which can be used to quantify the uncertainty associated with the estimated comparison probabilities.

The proposed modeling framework also needs to be extended to accommodate more complex settings of pairwise comparisons. First, 
covariate information can be incorporated into the model to facilitate the prediction. For example, for many team sports tournaments (e.g., soccer and basketball), whether a team plays at their home court matters and should be included as a covariate. Second,  pairwise comparison data are often collected over time, which is true for the \textit{StarCraft II} and tennis data studied in Section~\ref{sect: real data}. The current model ignores time information in data. To better predict future pairwise comparison results, it will be useful to model the comparison probabilities as a function of time. As a result, the estimation of these time-varying comparison probabilities will also differ substantially from the current procedure. Third, for pairwise comparison data produced by raters, which are commonly encountered in crowd-sourcing settings \citep[e.g.,][]{chen_Bennett_etal-2013-prooceedings}, characteristics of the raters, such as their reliability, affect the pairwise comparisons. In other words, the distribution of the comparison between two items depends not only on the pair of items but also on the rater who performs the comparison.  In this regard, \cite{chen_Bennett_etal-2013-prooceedings} propose an extended version of the BT model that uses a rater-specific latent variable to account for raters' reliability. A similar extension can be made to the current model to simultaneously account for the raters' heterogeneity and the items' stochastic intransitivity.

\section*{Acknowledgments}
The authors would like to thank the editor and the two anonymous reviewers for their constructive and valuable comments, which substantially improved the paper.
\newpage

\appendix
\section{\rev{Additional Theoretical Results}}\label{app: additional theoretical results}
\subsection{\rev{Assumptions and Refinement Procedure for Max-norm and Two-to-infinity Norm Convergence}}\label{app: Assumption and refinement}
\rev{In this section, we state the assumptions and refinement procedures that enable us to obtain max-norm convergence of the underlying parameter matrix. In particular, we impose the following assumptions. 
\begin{assumption}\label{assp: rank}
    $M^*$ has rank $2K$ such that $M^*= U^* \Sigma^* J_K{U^*}^{\top}$ with  $\|U^* {\Sigma^*}^{1/2}\|_{2 \to \infty}\leq C^*$ for some fixed constant $C^*>0$. Moreover, we assume each singular value has multiplicity exactly $2$. Let $\sigma^*_{1} \geq \sigma_2^* \dots \geq \sigma_{K}^*$ be the largest $K$ distinct singular values of $M^*$. We assume that $\sigma^*_{K} \asymp \sigma^*_{1} \asymp n$.
\end{assumption}
\begin{assumption}\label{assp: p_n with rank}
    Assumption \ref{assp: 3} holds with $p_n \gtrsim n^{-1/3}\log(n)$.
\end{assumption}
Assumption \ref{assp: rank} requires that $M^*$ has an exact low-rank structure, with $\|U^* {\Sigma^*}^{1/2}\|_{2 \to \infty}\leq C^*$ serving as a standard incoherence condition to prevent spiky low-rank matrices (see, e.g., \citealp{chen2020noisy,chen_Li-2024-JoMR}). It is straightforward to verify that Assumption \ref{assp: rank} implies Assumption \ref{assp: 2}, which only requires an approximate low-rank structure. Since $M^*$ is skew-symmetric, the decomposition $M^* = U^* \Sigma^* J_K {U^*}^{\top}$ always holds by singular value decomposition. The assumptions on distinct and comparable singular values are made for simplicity and can be relaxed.}

\rev{Assumption \ref{assp: p_n with rank} strengthens the sparsity requirement, requiring $p_n \gtrsim n^{-1/3}\log n$. This condition arises from the dependence between the estimator and the missingness pattern in the refinement step. It may be relaxed by adopting a data-splitting refinement procedure, as in \cite{chen_Li-2024-JoMR}, to remove this dependence. We leave this extension to future work.}
\rev{Under these assumptions, we define the estimator satisfying both the nuclear norm and max norm constraint as 
\begin{align}\label{eq: nuclear and max norm estimator}
    \hat{M} = \argmax_{M} \LL(M) \text{ subject to } \|M\|_{*} \leq C_n n \text{ and } \|M\|_{\infty} \leq C_n/\sqrt{2K}, M = - M^{\top}.
\end{align}
Similar to \eqref{eq: nuclear norm estimator}, the above optimisation problem can be solved by a projected line search algorithm, and the projection step can be modified following Section 4.1.3 of \cite{davenport_etal-2014} to accommodate the additional constraint $\|M\|_{\infty} \le C_n/\sqrt{2K}$. }

\rev{For any $n \times 2K$ matrix $\ThTh = (\thth_1, \dots, \thth_n)^{\top}$, define $P^{2,\infty}_{\tau}(\ThTh) = (\breve{\thth}_1, \dots, \breve{\thth}_n)^{\top}$, where $\breve{\thth}_i = \thth_i$ if $\|\thth_i\| \leq \tau$ and $\breve{\thth}_i = (\tau/\|\thth_i\|)\thth_i$ otherwise. We further refine $\hat{M}$ using the following refinement procedure to obtain an estimator $\tilde{M}$ that achieves max-norm convergence.}

\begin{algorithm}
    \caption{\rev{Refinement Procedure}} \label{alg: refinement procedure}
    \rev{\textbf{Input:Matrix of comparison outcomes $Y$, nuclear norm constraint parameter $\tau$, initial estimate $\hat{M}$ solving \eqref{eq: nuclear and max norm estimator}, two-to-infinity norm constraint parameter $\tau_2$.}
    \begin{algorithmic}
        \State  Compute the best rank-2K approximation of \( \hat{M} \) using its top \(2K\) singular values to compute $\hat{M}_K$ given by 
        \begin{align*}
            \hat{M}_K = \hat{U} \hat{\Sigma}J_K {\hat{U}}^{\top},
        \end{align*}
where \( J_K \) is a $2K \times 2K$ block diagonal matrix  with $2 \times 2$ blocks 
\begin{align*}
    \begin{pmatrix}
        0& 1\\ -1 &0
    \end{pmatrix}
\end{align*}
along the main diagonal.
\State Compute $\hat{\ThTh} = P^{2,\infty}_{\tau_2}(\hat{U} \hat{\Sigma}^{1/2})$. 
          \State Calculate $\tilde{\ThTh} = (\tilde{\thth}_{1},\tilde{\thth}_2, \dots, \tilde{\thth}_n)^{\top}$, where we define $\hat{\thth}_j^{(p)} = J_{K}\hat{\thth}_j$ for $j = 1, \dots, n$, and $\tilde{\thth}_i$ is obtained by solving the equation $n^{-1}\sum_{j\in [n], j \neq i}  \hat{\thth}^{(p)}_j\left\{ y_{ij} - n_{ij}g(\tilde{\thth}_i^{\top}\hat{\thth}^{(p)}_j )   \right\}=0$.
    \end{algorithmic}
    \textbf{Output:} $\tilde{M} = \tilde{\ThTh} J_{K}\tilde{\ThTh}^{\top} $.}
\end{algorithm}
\rev{The following theorem establishes the two-to-infinity norm convergence of $\tilde{\ThTh}$ and the max-norm convergence of $\tilde{M}$ obtained from Algorithm \ref{alg: refinement procedure}.}

\rev{\begin{theorem}\label{thm: 2 to infty}
    Under Assumptions \ref{assp: rank} and \ref{assp: p_n with rank}, with probability converging to $1$, we have 
    \begin{align}\label{eq: tilde Th convergence}
        \|\tilde{\ThTh}- \ThTh^*\|_{2 \to \infty}  \leq \kappa_6 n^{-1/4}p_n^{-3/4} \text{ and }
    \end{align}
    \begin{align}\label{eq: tilde M convergence}
        \|\tilde{M} - M^*\|_{\infty} \leq \kappa_7 n^{-1/4}p_n^{-3/4},
    \end{align}
    where $\kappa_6$, $\kappa_7>0$ are absolute constants.
\end{theorem}
An immediate consequence of Theorem \ref{thm: 2 to infty} is that the same entrywise rate also holds for $\tilde{\Pi} - \Pi^*$. In particular, we have
\begin{align}\label{eq: Pi max norm convergence}
     \|\tilde{\Pi} - \Pi^*\|_{\infty} \leq \kappa_8 n^{-1/4}p_n^{-3/4},
\end{align}
for some positive constant $\kappa_8$, with probability converging to $1$. 

In addition to the effect of the refinement procedure discussed above, the convergence rates in the theorem also depend on the Frobenius norm error of $\hat{M}$. Therefore, sharper rates may be obtained by improving the initial convergence rate of $\hat{M}$ through a more refined exploitation of the low-rank structure. We leave this direction for future research.}

\subsection{\rev{Recovery of Top-k Items}}\label{app: Recovery of Top-k Items}
\rev{For any subset of subjects $\CS \subseteq [n]$, we may define a score within that subset by
\begin{align*}
    r_{i,\CS}(\Pi^*) 
    \coloneqq 
    \frac{1}{|\CS|} \sum_{j\in \CS}\Pi^*_{ij}, 
    \qquad i \in \CS.
\end{align*}
This quantity can be interpreted as the probability that subject $i$ defeats an opponent drawn uniformly at random from $\CS$. 
For example, if $\CS$ represents the set of players participating in a particular competition, then $r_{i,\CS}(\Pi^*)$ measures the average strength of player $i$ relative to that group. 
When $\CS = [n]$, this reduces to the score considered in \cite{shah_wainwright-2018-JoMR}.

Based on these scores, we may rank the subjects in $\CS$ and study recovery of the top-$k$ set. 
Let $(1), \dots, (|\CS|)$ denote the indices of the subjects sorted in descending order according to $r_{i,\CS}(\Pi^*)$, and define the $k-$separation threshold as
\begin{align*}
    \Delta_{k,\CS} 
    \coloneqq 
    r_{(k),\CS}(\Pi^*) 
    - 
    r_{(k+1),\CS}(\Pi^*).
\end{align*}
The following theorem gives a sufficient condition for consistent identification of the top-$k$ set.

\begin{theorem}\label{thm: top k recovery}
Under Assumptions \ref{assp: rank} and \ref{assp: p_n with rank}, suppose that
\[
    \Delta_{k,\CS} 
    \ge 
    2\kappa_8 n^{-1/4}p_n^{-3/4},
\]
where $\kappa_8$ is defined in \eqref{eq: Pi max norm convergence}.
Then
\[
    \mathbb P\big( \tilde{\CS}_k \neq \CS_k^* \big) \to 0,
\]
where $\CS_k^*$ and $\tilde{\CS}_k$ denote the sets of the top-$k$ subjects in $\CS$ ranked according to 
$r_{i,\CS}(\Pi^*)$ and $r_{i,\CS}(\tilde{\Pi})$, respectively.
\end{theorem}

The result follows directly from the entrywise convergence 
$\|\tilde{\Pi} - \Pi^*\|_{\infty}$ established in \eqref{eq: Pi max norm convergence}. 
In particular, the required separation level is determined by the rate of this entrywise error. 
Consequently, if a sharper convergence rate for 
$\|\tilde{\Pi} - \Pi^*\|_{\infty}$ is available, the separation condition on $\Delta_{k,\CS}$ can be correspondingly weakened.
}

\section{\rev{Proof of Theoretical Results}}\label{app: Proof of Theoretical Results}
\rev{This section presents the proofs of the theoretical results, where Section \ref{app: Proof of Theorem thm: nuclear norm convergence} proves Theorem \ref{thm: nuclear norm convergence}, Section \ref{app: Proof of thm: lower bound} proves Theorem \ref{thm: lower bound}, Section \ref{app: proves of thm: 2 to infty} proves Theorem \ref{thm: 2 to infty}, Section \ref{app: proves of thm: top k recovery} proves Theorem \ref{thm: top k recovery} and Section \ref{app: proof of proposition} proves Proposition \ref{prop: skew preserving}.} Throughout this section, $\delta_0, \delta_1, \dots $ denote positive constants that do not depend on $n$. For two probability distributions $\PP$ and $\QQ$ on a finite set $A$, $D(\PP \| \QQ)$ will denote the Kullback-Leibler (KL) divergence, 
\begin{align*}
    D(\PP \| \QQ) = \sum_{x \in A}\PP(x) \log\left(\frac{\PP(x)}{\QQ(x)}\right). 
\end{align*}
\subsection{Proof of Theorem \ref{thm: nuclear norm convergence}}\label{app: Proof of Theorem thm: nuclear norm convergence}
For two scalars $x, z \in [0,1]$, define the Hellinger distance as
\begin{align*}
    d^2_H(x,z) = (\sqrt{x} - \sqrt{z})^2 + (\sqrt{1-x} - \sqrt{1-z})^2.
\end{align*}
For $n \times n $ matrices $X = (x_{ij})_{n \times n}$ and $Z = (z_{ij})_{n \times n}$ where $X, Z \in [0,1]^{n \times n }$, define 
\begin{align*}
    d^2_H(X,Z) = \frac{1}{n^2} \sum_{i=1}^{n} \sum_{j=1}^{n}   d^2_H(x_{ij},z_{ij}).
\end{align*}
It is straightforward to show that $d^2_H(X,Z) \gtrsim \|X - Z\|_{F}^2/(n^2-n)$. Moreover, let $\|X\|_{\infty} = \max_{i,j}|x_{ij}|$ denotes the entry-wise infinity norm of $X$. We will first prove the theorem under an additional constraint that $\|M^{*}\|_{\infty} \leq \gamma$ and $\|\hat{M}\|_{\infty} \leq \gamma$ for some $\gamma >0$, then send $\gamma \to \infty $ to recover Theorem \ref{thm: nuclear norm convergence}. Formally, we prove the following theorem: 
\begin{theorem}\label{thm: restricted nuclear norm convergence}
    Under the conditions in Theorem \ref{thm: nuclear norm convergence}, suppose in addition that $\|M^{*}\|_{\infty} \leq \gamma$. Let $\hat{M}$ be a solution to \eqref{eq: nuclear norm estimator} under the additional constraint that $\|\hat{M}\|_{\infty} \leq \gamma$. Then with probability at least $1 - \delta_1/n$, 
    $$ d^2_H(\hat{\Pi},\Pi^{*}) \leq \delta_2 C_n\sqrt{\frac{1}{p_n n}},$$
    where $\delta_1$ and $\delta_2$ are absolute constants.
\end{theorem}
\begin{proof}
Define $\bar{\LL}(M) = \LL(M) - \LL(0_{n \times n })$. The following lemma is essential to proving Theorem \ref{thm: restricted nuclear norm convergence}:
\begin{lemma} \label{lm: lemma for nuclear norm thm}
    Under the conditions in Theorem \ref{thm: restricted nuclear norm convergence}, we have 
    \begin{align*}
       P\left( \frac{1}{n^2}\sup_{M \in \GG}|\bar{\LL}(M) - E (\bar{\LL}(M)) | \geq \delta_0 C_n\sqrt{\frac{T q_n}{n}}  \right) \leq \frac{\delta_1}{n}, 
    \end{align*}
    where $\delta_0$ is an absolute constant, and $\GG \subset \RR^{n \times n}$ is defined as 
    \begin{align*}
        \GG = \{M \in \RR^{n \times n} : \|M\|_{*} \leq C_n n ,\|M\|_{\infty} \leq \gamma, M  = -M^{\top}\}.
    \end{align*}
\end{lemma}
Before proving the lemma, we first show how Lemma \ref{lm: lemma for nuclear norm thm} implies Theorem \ref{thm: restricted nuclear norm convergence}. For two scalars $x, z \in [0,1]$, we abuse the notation of $D(\cdot \| \cdot)$ and define the divergence measure as
\begin{align*}
    D(x \| z) = x\log\left(\frac{x}{z}\right) + (1-x)\log\left(\frac{1-x}{1-z}\right).
\end{align*}
Similarly, for two matrices $X, Z \in [0,1]^{n \times n}$, define
\begin{align*}
    D(X \| Z ) = \sum_{i=1}^{n} \sum_{j=1}^{n} D(x_{ij} \| z_{ij}). 
\end{align*}
For any choice of $M \in \GG$, we have 
\begin{align*}
    &E(\bar{\LL}(M) -\bar{\LL}(M^{*})) \\
    =&  E(\LL(M) -\LL(M^{*}))\\
                                        =&  \sum_{i=1}^{n} \sum_{j > i} E\left(y_{ij}\log\left(\frac{g( m_{ij} )}{g( m_{ij}^{*} )}\right)  + (n_{ij} - y_{ij}) \log\left( \frac{1-g( m_{ij} )}{1-g( m_{ij}^{*} )}\right) \right)\\
                                            =& \sum_{i=1}^{n} \sum_{j > i} E\left(n_{ij}g( m_{ij}^{*} ) \log\left(\frac{g( m_{ij} )}{g( m_{ij}^{*} )}\right)  + n_{ij}(1- g( m_{ij}^{*} ))\log\left( \frac{1-g( m_{ij} )}{1-g( m_{ij}^{*} )}\right) \right)\\
                                                 =& -T \sum_{i=1}^{n} \sum_{j > i} p_{ij,n}D(g( m_{ij}^{*}) \| g( m_{ij} ))\\
                                                \leq&  -0.5Tp_n  D(\Pi^{*} \| \Pi ) . 
\end{align*}
Note that $M^{*} \in \GG$ by assumption. Therefore, for any $M \in \GG$, we have 
\begin{align*}
    \bar{\LL}(M) -\bar{\LL}(M^{*}) & =  E(\bar{\LL}(M) -\bar{\LL}(M^{*})) + (\bar{\LL}(M) - E(\bar{\LL}(M))) - ( \bar{\LL}(M^{*}) - E(\bar{\LL}(M^{*})) )\\
                                     & \leq E(\bar{\LL}(M) -\bar{\LL}(M^{*}))  + 2 \sup_{X \in \GG}|\bar{\LL}(X) - E(\bar{\LL}(X))|\\
                                     &\leq -0.5Tp_n  D(\Pi^{*} \| \Pi )+  2 \sup_{X \in \GG}|\bar{\LL}(X) - E(\bar{\LL}(X))|.
\end{align*}
Moreover, from the definition of $\hat{M}$, we have $\hat{M} \in \GG$ and $\LL(\hat{M}) \geq \LL(M^{*}).$ Therefore, we obtain 
\begin{align*}
    0 \leq-0.5Tp_n  D(\Pi^{*} \| \hat{\Pi} ) +  2 \sup_{M \in \GG}|\bar{\LL}(M) - E(\bar{\LL}(M))|.
\end{align*}
Applying Lemma \ref{lm: lemma for nuclear norm thm}, then with probability at least $1 - \delta_1/n$, we have 
\begin{align*}
    0 \leq \frac{-0.5Tp_n D(\Pi^{*} \| \hat{\Pi} )}{n^2} +  2\delta_0 C_n\sqrt{\frac{T q_n}{n}} .
\end{align*}
This implies that 
\begin{align*}
                                        \frac{D(\Pi^{*} \| \hat{\Pi} )}{n^2}    &\leq \frac{4\delta_0 C_n}{Tp_n}\sqrt{\frac{T q_n}{n}}
                                                                      \lesssim \frac{4\delta_0 C_n}{\sqrt{Tp_n}} \sqrt{\frac{1}{n}}
\end{align*}
by Assumption \ref{assp: 3}. Note that $d_H^2(\hat{\Pi}, \Pi^{*} ) \leq n^{-2}D(\Pi^* \|  \hat{\Pi})$ by Jensen's inequality combined with the fact that $(1-x) \leq \log(x)$. Hence Theorem \ref{thm: restricted nuclear norm convergence} is proved. Theorem \ref{thm: nuclear norm convergence} then follows by the fact that $d_H^2(\hat{\Pi}, \Pi^{*} ) \gtrsim \|\hat{\Pi} - \Pi^{*}\|_{F}^2/(n^2 -n)$ and taking the limit as $\gamma \to \infty$.
\end{proof}
We now begin to prove Lemma \ref{lm: lemma for nuclear norm thm}. 
\begin{proof}
     For any $h >0$, using Markov's inequality, we have 
\begin{align}\label{eq: eq 1 for proof of nuclear norm lemma}
    &P\left(\frac{1}{n^2} \sup_{M \in \GG}|\bar{\LL}(M) - E (\bar{\LL}(M)) | \geq \delta_0 C_n\sqrt{T q_n/n}  \right) \nonumber\\
    =& P\left( \sup_{M \in \GG} |\bar{\LL}(M) - E (\bar{\LL}(M)) |^{h} \geq \left(\delta_0 C_n n^{1.5} \sqrt{T q_n}  \right)^{h} \right)\nonumber \\
    \leq & \frac{E\left( \sup_{M \in \GG}|\bar{\LL}(M) - E (\bar{\LL}(M)) |^{h}   \right) }{\left(\delta_0 C_n n^{1.5}\sqrt{T q_n}  \right)^{h}}.
\end{align}
The bound in Lemma \ref{lm: lemma for nuclear norm thm} will be established by combining \eqref{eq: eq 1 for proof of nuclear norm lemma}, deriving an upper bound on $E\left( \sup_{M \in \GG}|\bar{\LL}(M) - E (\bar{\LL}(M)) |^{h}   \right)$ and setting $h = \log(n)$. Note that we can write $\bar{\LL}(M)$ as 
\begin{align*}
    \bar{\LL}(M) =  \sum_{i=1}^{n} \sum_{j >i } y_{ij}\log\left(\frac{g( m_{ij} )}{g( 0 )}\right)  + (n_{ij} - y_{ij}) \log\left( \frac{1-g( m_{ij} )}{1-g( 0 )}\right). 
\end{align*}
By a symmetrization argument (Lemma 6.3 in \cite{ledoux_talagrand-1991probability}), we have
\begin{align*}
    &E\left( \sup_{M \in \GG}|\bar{\LL}(M) - E (\bar{\LL}(M)) |^{h}   \right) \\
    \leq& 2^h E\left( \sup_{M \in \GG} \left| \sum_{i=1}^{n} \sum_{j > i} \epsilon_{ij} \left\{y_{ij}\log\left(\frac{g( m_{ij} )}{g( 0 )}\right)  + (n_{ij} - y_{ij}) \log\left( \frac{1-g( m_{ij} )}{1-g( 0 )}\right) \right\}   \right|^{h} \right),
\end{align*}
where $\epsilon_{i,j}$ are i.i.d. Rademacher random variables for $i, j = 1, \dots, n$. To bound the latter term, we apply a contraction principle (Theorem 4.12 in \cite{ledoux_talagrand-1991probability}). From the assumption that $\|M\|_{\infty} \leq \gamma$, conditional on $n_{ij}$, for $n_{ij} \geq 1$, 
\begin{align*}
    n_{ij}^{-1}\left(y_{ij}\log\left(\frac{g( m_{ij} )}{g( 0 )}\right)  +(n_{ij} - y_{ij})\log\left(\frac{1-g( m_{ij} )}{1-g( 0 )}\right) \right) 
\end{align*}
is a contraction that vanish at $0$. Thus, we have 
\begin{align}\label{eq: contraction outcome}
    E\left( \sup_{M \in \GG}|\bar{\LL}(M) - E (\bar{\LL}(M)) |^{h}   \right)\leq&(2^h)(2^h) E\left( \sup_{M \in \GG} \left|  \sum_{i=1}^{n} \sum_{j > i} n_{ij} \epsilon_{ij} m_{ij}   \right|^{h} \right)\nonumber\\
    =& 4^{h} E\left( \sup_{M \in \GG} \left|  \sum_{i=1}^{n} \sum_{j > i}n_{ij} \epsilon_{ij} m_{ij}   \right|^{h} \right). 
\end{align}
To bound $ E\left( \sup_{M \in \GG} \left|  \sum_{i=1}^{n} \sum_{j > i} n_{ij}\epsilon_{ij} m_{ij}   \right|^{h} \right)$, we apply the skew-symmetric property of $M$ and the fact that $n_{ij} = n_{ji}$ for $i,j \in \{1, \dots, n\}$. For any $M \in \GG$, we have 
\begin{align*}
    \sum_{i=1}^{n} \sum_{j =1}^{n} n_{ij}\epsilon_{ij} m_{ij} = \sum_{i=1}^{n} \sum_{j >i} n_{ij}(\epsilon_{ij} - \epsilon_{ji}) m_{ij}.  
\end{align*}
On the other hand, for $h >1$, by the convexity of $|\cdot|^{h}$, we have 
\begin{align*}
    \left|  \sum_{i=1}^{n} \sum_{j > i} n_{ij} \epsilon_{ij} m_{ij}   \right|^h  &=   \left| 0.5\left\{\sum_{i=1}^{n} \sum_{j > i} n_{ij}(\epsilon_{ij} -\epsilon_{ji}) m_{ij}  +  \sum_{i=1}^{n} \sum_{j > i} n_{ij}(\epsilon_{ij} + \epsilon_{ji}) m_{ij} \right\}  \right|^h\\
    & \leq  0.5\left(  \left|\sum_{i=1}^{n} \sum_{j > i} n_{ij}(\epsilon_{ij} -\epsilon_{ji}) m_{ij} \right|^h +\left|  \sum_{i=1}^{n} \sum_{j > i} n_{ij}(\epsilon_{ij} + \epsilon_{ji}) m_{ij}  \right|^h\right).
\end{align*}
Since $\epsilon_{ji}$ and $-\epsilon_{ji}$ have identical distribution, after taking expectation, we have 
\begin{align} \label{eq: Bound on expectation}
    E\left( \sup_{M \in \GG} \left|  \sum_{i=1}^{n} \sum_{j > i}n_{ij} \epsilon_{ij} m_{ij}   \right|^{h} \right) &\leq E\left( \sup_{M \in \GG} \left|\sum_{i=1}^{n} \sum_{j > i} n_{ij}(\epsilon_{ij} -\epsilon_{ji}) m_{ij} \right|^h\right)\nonumber\\
    &=  E\left( \sup_{M \in \GG} \left| \sum_{i=1}^{n} \sum_{j =1}^{n} n_{ij}\epsilon_{ij} m_{ij}\right|^{h}\right)\nonumber\\
    &= E\left( \sup_{M \in \GG} \left|  \langle \EE \circ \NN, M \rangle\right|^{h} \right).
\end{align}
 Here, $\EE = (\epsilon_{ij})_{n \times n}$, $\NN = (n_{ij})_{n \times n}$ and $\EE \circ \NN$ represents the hadamard product between $\EE$ and $\NN$, and  $\langle X, Z \rangle = \sum_{i=1}^{n} \sum_{j=1}^{n} x_{ij}z_{ij}$ for any $n \times n$ matrices $X$ and $Z$. Note that $|\langle X, Z\rangle| \leq \|X\|_{op} \|Z\|_{*}$, where $\|\cdot\|_{op}$ is the Euclidean operator norm. Hence we have 
\begin{align}\label{eq: operator norm bound}
    E\left( \sup_{M \in \GG} \left|  \langle \EE \circ \NN, M \rangle\right|^{h} \right) & \leq E\left( \sup_{M \in \GG} \|\EE \circ \NN\|_{op}^{h}\|M \|_{*}^{h}\right)\nonumber\\
                                                                                 &  \leq (C_n n)^{h} E\left(  \|\EE \circ \NN\|_{op}^{h}\right).
\end{align}
We can write $\EE \circ \NN = \sum_{i=1}^{n}\sum_{j=1}^{n} \epsilon_{ij}n_{ij} E_{ij}$, where $E_{ij}$ is a $n \times n $ matrix with $1$ at the $(i,j)$th entry and $0$ otherwise.  Following arguments similar to Section 4.3 of \cite{tropp2015introduction}, and applying Theorem 4.1.1 in \cite{tropp2015introduction}, for $t >0$, we set $ s = -t^{2/h}/(2 \max_{j}\{\sum_{i=1}^{n}n_{ij}^2\})$ such that
 \begin{align*}
     E( \| \EE \circ \NN  \|^{h}_{op} \mid \NN)
     =&\left(\int_{0}^{\infty} P(  \|\EE \circ \NN\|^{h}_{op}  \geq t)dt\right)\\ 
     \leq& \left( \int_{0}^{\infty}      2n\exp\left(\frac{-t^{2/h}}{2\max_{j}\{\sum_{i=1}^{n}n_{ij}^2\}}\right) dt\right)\\
      =& \left( \int_{0}^{\infty}     2n \left(\frac{h}{2}(2 \max_{j}\{\sum_{i=1}^{n}n_{ij}^2\} )^{h/2} s^{h/2-1} \right) \exp\left(-s \right) ds\right)\\
      =& \left((nh)(2 \max_{j}\{\sum_{i=1}^{n}n_{ij}^2 \})^{h/2}  \int_{0}^{\infty}     s^{h/2-1}  \exp\left(-s \right) ds\right)\\
      =&  nh\Gamma(h/2) (2 \max_{j}\{\sum_{i=1}^{n}n_{ij}^2 \})^{h/2} ,
 \end{align*}
 where $\Gamma(\cdot)$ is the gamma function. Taking expectation, we have 
 \begin{align}\label{eq: bound of E( E circ N)}
          E( \| \EE \circ \NN  \|^{h}_{op})  \leq  nh\Gamma(h/2)2^{h/2}E( \max_{j}\{\sum_{i=1}^{n}n_{ij}^2 \}^{h/2}).
 \end{align}
 We aim to find a bound for $E(\max_{j}\{\sum_{i=1}^{n}n_{ij}^2 \}^{h/2}).$ Using Bernstein's inequality, for each $j$ and all $t>0$, we have 
 \begin{align*}
     P\left( \left|\sum_{i=1}^{n} \left( n_{ij}^2 -E(n_{ij}^2) \right) \right| >t \right) &\leq 2 \exp\left(\frac{-t^2/2}{ \sum_{i=1}^{n}\{E(n_{ij}^4) - (E(n_{ij}^2))^2\} + T^2 t/3} \right)\\
                                                         &\leq 2 \exp\left(\frac{-t^2/2}{ nT^4q_n + T^2 t/3} \right). 
 \end{align*}
 In particular, for $t \geq 6 n T^2 q_n,$ we have
 \begin{align*}
      P\left(\left|\sum_{i=1}^{n} \left( n_{ij}^2 -E(n_{ij}^2) \right)  \right| >t \right) &\leq 2 \exp\left( -t/T^2 \right) = 2P(U_j > t/T^2), 
 \end{align*}
 where $U_1, \dots, U_n$ are independent and identically distributed exponential random variables. Hence, we have 
 \begin{align*}
   &\left(  E\left(\max_{j}\left \{\sum_{i=1}^{n}n_{ij}^2 \right \}^{h/2}\right) \right)^{1/h}\\
   =&\left(  E\left(\max_{j}\left |\sum_{i=1}^{n}n_{ij}^2 -E(n_{ij}^2) + E(n_{ij}^2) \right |^{h/2}\right) \right)^{1/h}\\
   \leq& 2\left(E\left(\max_{j}\left |\sum_{i=1}^{n}n_{ij}^2 -E(n_{ij}^2)\right|^{h/2} \right) \right)^{1/h}+2 \left(E\left(\max_{j}\left| \sum_{i=1}^{n} E(n_{ij}^2) \right |^{h/2}\right) \right)^{1/h}\\
   \leq& 2\sqrt{nT^2q_n} + 2 \left(E\left(\max_{j}\left |\sum_{i=1}^{n}n_{ij}^2 -E(n_{ij}^2)\right|^{h} \right) \right)^{1/2h} \\
   =& 2\sqrt{nT^2q_n} + 2 \left( \int_{0}^{\infty} P\left( \max_{j}\left |\sum_{i=1}^{n}n_{ij}^2 -E(n_{ij}^2)\right|^{h} \geq t   \right)   dt \right)^{1/2h}\\
   \leq & 2\sqrt{nT^2q_n} + 2 \left\{ (6nT^2q_n)^h +  \int_{(6nT^2q_n)^h}^{\infty} P\left( \max_{j}\left |\sum_{i=1}^{n}n_{ij}^2 -E(n_{ij}^2)\right|^{h} \geq t   \right)   dt \right\}^{1/2h}\\
   \leq & 2\sqrt{nT^2q_n} + 2 \left\{ (6nT^2q_n)^h +  2 \int_{(6nT^2q_n)^h}^{\infty} P\left( \max_{j} \{U_{j}\}^{h} \geq t/T^{2h}   \right)   dt \right\}^{1/2h}\\
   \leq & 2\sqrt{nT^2q_n} + 2 \left\{ (6nT^2q_n)^h +  2 E\left( \max_{j} \{T^{2} U_{j}\}^{h} \right)  \right\}^{1/2h}\\
   =&2\sqrt{nT^2q_n} + 2 \left\{ (6nT^2q_n)^h +  2 T^{2h}E\left( (\max_{j} \{ U_{j}\})^{h} \right)  \right\}^{1/2h}.
 \end{align*}
 By standard computations for exponential random variables, we can obtain the inequality $E\left( (\max_{j} \{ U_{j}\})^{h} \right) \leq 2h! + \log^{h}(n)$. Thus, we have 
 \begin{align*}
     \left(  E\left(\max_{j}\left \{\sum_{i=1}^{n}n_{ij}^2 \right \}^{h/2}\right) \right)^{1/h}
     \leq& 2\sqrt{nT^2q_n} + 2 \left\{ (6nT^2q_n)^h +  2 T^{2h}(2h! + \log^{h}(n) )\right\}^{1/2h}\\
     \leq&2T(1 + \sqrt{6})\sqrt{nq_n} + 2T (2)^{1/2h}(\sqrt{\log(n)} + 2^{1/2h}\sqrt{h})\\
      \leq & 2T(1 + \sqrt{6})\sqrt{nq_n} + 2T(2 + \sqrt{2})\sqrt{\log(n)}
 \end{align*}
 using the choice $h = \log(n)$ in the final line. Combining this result with \eqref{eq: bound of E( E circ N)}, we have 
\begin{align*}
     E( \| \EE \circ \NN  \|^{h}_{op})^{1/h}  &\leq  (nh\Gamma(h/2))^{1/h} \sqrt{2}\{ 2T(1 + \sqrt{6})\sqrt{nq_n} + 2T(2 + \sqrt{2})\sqrt{\log(n)}\}\\
                                                &\leq  \delta_3 T \sqrt{nq_n}
\end{align*}
for some constant $\delta_3 >0$ by Assumption \ref{assp: 3}. Combining this with \eqref{eq: contraction outcome}, \eqref{eq: Bound on expectation} and \eqref{eq: operator norm bound}, we obtain 
\begin{align*}
    E\left( \sup_{M \in \GG}|\bar{\LL}(M) - E (\bar{\LL}(M)) |^{h}   \right)^{1/h} \leq  \left(4T\right)(C_n n)(\delta_3) \sqrt{nq_n}.
\end{align*}
Plugging this into \eqref{eq: eq 1 for proof of nuclear norm lemma}, the probability in \eqref{eq: eq 1 for proof of nuclear norm lemma} is upper bounded by 
\begin{align*}
  \left\{  \frac{\left(4T\right)(C_n n) (\delta_3) \sqrt{nq_n}}{\delta_0 C_n n^{1.5} \sqrt{T q_n}  } \right\}^{h} 
  \leq\left(\frac{4 \sqrt{T}\delta_3 }{\delta_0} \right)^{\log(n)}
\leq\frac{\delta_1}{n},
\end{align*}
 provided that $\delta_0 \geq 4 \sqrt{T}\delta_3 /e$, which establishes the lemma. 
\end{proof}

\begin{remark}\label{rmk: app relax assmp}
\rev{The assumption $p_n \asymp q_n$ can be relaxed based on the above proof. In particular, without imposing this condition, one can show that with probability at least $1 - \kappa_1/n$,
\begin{align*}
    \frac{1}{n^2 - n}\, \|\hat{\Pi} - \Pi^{*}\|_F^2 
    \;\le\; \kappa_2\, C_n \sqrt{\frac{q_n}{p_n^2 n}},
\end{align*}
under the remaining assumptions of Theorem~\ref{thm: nuclear norm convergence}. If $T$ diverges rather than remains fixed, a similar argument, following the proof of Theorem~\ref{thm: restricted nuclear norm convergence} using the rate established in Lemma~\ref{lm: lemma for nuclear norm thm}, yields
\begin{align*}
     \frac{1}{n^2 - n}\, \|\hat{\Pi} - \Pi^{*}\|_F^2 
    \;\le\; \kappa_2\, C_n \sqrt{\frac{1}{T p_n n}}
\end{align*} with probability converging to one. Under these relaxed settings, it is unclear whether the rate is optimal, and we leave this question for future work.} 
\end{remark}
\subsection{Proof of Theorem \ref{thm: lower bound}}\label{app: Proof of thm: lower bound}
We first quote the following lemma from \cite{davenport_etal-2014}:
\begin{lemma}\label{lm: lemma for Dxy}
    Suppose $x, z \in (0,1)$. Then 
    \begin{align*}
        D(x \| z) \leq \frac{(x-z)^2}{z(1-z)}. 
    \end{align*}
\end{lemma}
The following lemma constructs a packing set $\XX \subset \KK$ such that, for any distinct $X^{(a)}, X^{(b)} \in \XX$, $\|X^{(a)}  - X^{(b)}\|_{F}^2$ is large: 
\begin{lemma}\label{lm: lemma for lower bound}
    Let $\KK$ be defined as in \eqref{eq: defintion of KK}, and $k$ a positive integer. Let $\gamma \leq 1$ be such that $k/\gamma^2$ is an integer, and suppose $k/\gamma^2 \leq n$. Then, there exists a set $\XX \subset \KK$ satisfying
    $$ |\XX| \geq \exp\left(\frac{kn}{25600\gamma^2} \right) $$
    with the following properties: 
    \begin{enumerate}
        \item For all $X = (x_{ij})_{n \times n} \in \XX$, each entry of $X$ satisfies $|x_{ij}| \leq    C_n\gamma /\sqrt{2k}$.  
        \item For all $X^{(a)} \neq  X^{(b)} \in \XX$,  
        \begin{align*}
            \|X^{(a)} - X^{(b)} \|^2_{F} > \frac{C_n^2 \gamma^2 n^2}{16k}. 
        \end{align*}
    \end{enumerate}
\end{lemma}
\begin{proof}
     We use a probabilistic argument. The set will be constructed by drawing 
\begin{align*}
    |\XX| = \left\lceil\exp\left(\frac{kn}{25600\gamma^2}\right) \right\rceil
\end{align*}
matrices independently from the following distribution. Set $B = k/\gamma^2$. Each matrix in $\XX$ is constructed of the form $S - S^{\top}$, where $S = (s_{ij})_{n \times n}$ consists of blocks of dimension $B \times n$, stacked vertically. The entries of the first block are independent and identically distributed symmetric random variables taking values $\pm C_n\gamma /(2\sqrt{2k})$. Then $S$ is filled out by copying this block as many times as it fits. That is, 
\begin{align*}
    s_{ij} = s_{i^{'}j}, \text{ where } i^{'} = i ~(\text{mod } B) +1.
\end{align*}
Now we argue that with nonzero probability, this set will have all the desired properties. For $X \in \XX$, it is easy to verify that $X= -X^{\top}$. Moreover, we have
\begin{align*}
    \|X\|_{\infty} \leq 2\{C_n\gamma /(2\sqrt{2k}) \}  \leq C_n /\sqrt{2k}. 
\end{align*}
Further, since rank$(X)  \leq 2\text{rank}(S)  \leq 2B$, 
\begin{align*}
    \|X \|_{*} \leq \sqrt{2B} \| X \|_{F} \leq \sqrt{2k/\gamma^2}  n (C_n\gamma/\sqrt{2k}) = C_n n. 
\end{align*}
Thus $\XX \subset \KK $, and it remains to show that $\XX$ satisfies property 2 in Lemma \ref{lm: lemma for lower bound}. Let $p = \lfloor n/B \rfloor$. Consider the submatrix of $S$ containing the first $B$ rows, denoted by $S_{[1:B,:]}$. This can be written as 
\[
S_{[1:B,:]} = (S_1, S_2, \dots, S_p, S_{p+1}),
\]
where $S_1, \dots, S_p$ are matrices of dimension $B \times B$, and $S_{p+1}$ accounts for the remaining part of $S_{[1:B,:]}$. If $n$ is divisible by $B$, then $S_{p+1}$ is an empty matrix. For $X^{(a)} = S^{(a)} - (S^{(a)})^{\top}$ and $X^{(b)} = S^{(b)} - (S^{(b)})^{\top}$, drawn from the above distribution, define 
\[
\Theta_i = \frac{\sqrt{2k}}{C_n \gamma} \left(S_i^{(a)} - S_i^{(b)}\right), \quad \text{for } i = 1, \dots, p.
\]
Each $\Theta_i$ is a $B \times B$ matrix, and we write $\Theta_i = (\theta_{i,sl})_{B \times B}$, where each $\theta_{i,sl}$ is independent and identically distributed random variables such that for each $s,l \in \{1, \dots, B\}, $ we have 
\[
P(\theta_{i,sl} =  1)=P(\theta_{i,sl} =  -1)  = 0.25 \text{ and } P(\theta_{i,sl} = 0) = 0.5.
\]
Hence we can write
\begin{align*}
    \|X^{(a)} - X^{(b)} \|^2_{F} &\geq \sum_{i=1}^{p} \sum_{j=1}^{p} \|S_i^{(a)} - (S_j^{(a)})^{\top} - S_i^{(b)} + (S_j^{(b)})^{\top} \|^{2}_{F}\\
                                 & =\frac{C_n^2 \gamma^2}{2k} \sum_{i=1}^{p} \sum_{j=1}^{p} \|\Theta_i  - \Theta_j^{\top}\|^{2}_{F}\\
                                 & = \frac{C_n^2 \gamma^2}{2k} \sum_{i=1}^{p} \sum_{j=1}^{p} (\|\Theta_i\|^2_{F} +   \|\Theta_j^{\top}\|^{2}_{F} - 2tr(\Theta_i\Theta_j )) \\
                                 & = \frac{C_n^2 \gamma^2}{2k} \left\{2p\sum_{i=1}^{p}(\|\Theta_i\|^2_{F})  - 2tr\left((\sum_{i=1}^{p} \Theta_i) (\sum_{i=1}^{p}\Theta_i) \right)\right\}. 
\end{align*}
The trace can be expanded as:
\begin{align*}
    tr\left((\sum_{i=1}^{p} \Theta_i) (\sum_{i=1}^{p}\Theta_i)\right)&= \sum_{s=1}^{B} \sum_{k=1}^{B} (\sum_{i=1}^{p} \theta_{i,sl} )(\sum_{i=1}^{p} \theta_{i,ls} )\\
    & = 2\sum_{s=1}^{B} \sum_{k>s} (\sum_{i=1}^{p} \theta_{i,sl} )(\sum_{i=1}^{p} \theta_{i,ls} ) + \sum_{s=1}^{B}(\sum_{i=1}^{p} \theta_{i,ss} )^2. 
\end{align*}
Hence we can write 
\begin{align*}
    &2p\sum_{i=1}^{p}(\|\Theta_i\|^2_{F})  - 2tr\left((\sum_{i=1}^{p} \Theta_i) (\sum_{i=1}^{p}\Theta_i) \right)\\
     =& 2p\sum_{i=1}^{p}\sum_{s=1}^{B}\sum_{k=1}^{B} \theta_{i,sl}^2 - 4\sum_{s=1}^{B} \sum_{k>s} (\sum_{i=1}^{p} \theta_{i,sl} )(\sum_{i=1}^{p} \theta_{i,ls} ) - 2\sum_{s=1}^{B}(\sum_{i=1}^{p} \theta_{i,ss} )^2\\
     =& 2\sum_{s=1}^{B} \{ p\sum_{i=1}^{p} (\theta_{i,ss}^2) - (\sum_{i=1}^{p} \theta_{i,ss} )^2 \} + 2 \sum_{s=1}^{B}\sum_{k > s}^{B}\{ p\sum_{i=1}^{p}(\theta_{i,sl}^2 +\theta_{i,ls}^2  ) - 2(\sum_{i=1}^{p} \theta_{i,sl} )(\sum_{i=1}^{p} \theta_{i,ls} ) \}.
\end{align*}
Taking expectation, we have
\begin{align*}
E\left(2p\sum_{i=1}^{p}(\|\Theta_i\|^2_{F})  - 2tr\left((\sum_{i=1}^{p} \Theta_i) (\sum_{i=1}^{p}\Theta_i) \right)\right)
    =&2B \{p(0.5p) - 0.5p  \} + \frac{2B(B-1)p^2 }{2}\\
    =& Bp^2 - Bp + B^2p^2 -Bp^2\\
     =&B^2p^2  -  Bp . 
\end{align*}
Using the fact that $p = \lfloor n/B \rfloor \geq n/2B$ and $p \leq n/B$, we have 
\begin{align*}
    &P\left(2p\sum_{i=1}^{p}(\|\Theta_i\|^2_{F})  - 2tr\left((\sum_{i=1}^{p} \Theta_i) (\sum_{i=1}^{p}\Theta_i) \right) \leq n^2/8 \right)\\
=& P\left(-2p\sum_{i=1}^{p}(\|\Theta_i\|^2_{F})  + 2tr\left((\sum_{i=1}^{p} \Theta_i) (\sum_{i=1}^{p}\Theta_i) \right) + B^2p^2  - Bp\geq -n^2/8  + B^2p^2  - Bp \right)\\
\leq &P\left(-2p\sum_{i=1}^{p}(\|\Theta_i\|^2_{F})  + 2tr\left((\sum_{i=1}^{p} \Theta_i) (\sum_{i=1}^{p}\Theta_i) \right) + B^2p^2  - Bp\geq -n^2/8  + n^2/4  - n \right)\\
\leq & P\left(-2p\sum_{i=1}^{p}(\|\Theta_i\|^2_{F})  + 2tr\left((\sum_{i=1}^{p} \Theta_i) (\sum_{i=1}^{p}\Theta_i) \right) + B^2p^2  - Bp\geq n^2/16 \right),
\end{align*}
where the last inequality holds as long as $n \geq 16$. 
Using McDiarmid's inequality, we can obtain the bound 
\begin{align*}
    P\left(2p\sum_{i=1}^{p}(\|\Theta_i\|^2_{F})  - 2tr\left((\sum_{i=1}^{p} \Theta_i) (\sum_{i=1}^{p}\Theta_i) \right) \leq n^2/8 \right) &\leq \exp\left( -\frac{2(n^2/16)^2 }{\sum_{s=1}^{B}\sum_{k=1}^{B}\sum_{i=1}^{p}(10p)^2 }\right)\\
    & =\exp\left( -\frac{n^4 }{ 12800B^2 p^3 }\right)\\
    & \leq \exp\left( -\frac{nB }{ 12800 }\right). 
\end{align*}
Using Union bound, we have that 
\begin{align*}
    P\left(\min_{X^{(a)} \neq X^{(b)} \in \XX} 2p\sum_{i=1}^{p}(\|\Theta_i\|^2_{F})  - 2tr\left((\sum_{i=1}^{p} \Theta_i) (\sum_{i=1}^{p}\Theta_i) \right) \leq n^2/8 \right)\leq \binom{|\XX|}{2} \exp\left( -\frac{nB }{ 12800 }\right),
\end{align*}
which is less than $1$ given the size of $\XX$. Thus the event that 
$$2p\sum_{i=1}^{p}(\|\Theta_i\|^2_{F})  - 2tr\left((\sum_{i=1}^{p} \Theta_i) (\sum_{i=1}^{p}\Theta_i) \right) > n^2/8$$ for all $X^{(a)} \neq X^{(b)} \in \XX$ has non-zero probability. In this event, 
\begin{align*}
    \|X^{(a)} - X^{(b)} \|^2_{F} > \frac{C_n^2 \gamma^2}{2k}(n^2/8) = \frac{C_n^2 \gamma^2n^2}{16k}.
\end{align*}
The proof of the lemma is thus complete. 
\end{proof}

We now proceed to prove the following theorem, which concerns the lower bound treating $n_{ij}$ as given. 
\begin{theorem}\label{thm: lower bound nij given}
    Suppose $ 12 \leq C_n^2 \leq \min\{1, (\kappa^{'}_3)^2/T\}n$. For any given $n_{ij}$, $i,j \in \{1, \dots, n\}, j > i$, consider any algorithm which, for any $M\in \KK$, takes as input $Y$ and returns $\hat{M}$. Then there exists $M \in \KK$ such that with probability at least 3/4, $\Pi = g(M)$ and $\hat{\Pi} = g(\hat{M})$ satisfy
    \begin{align}\label{eq:temp Pi lower bound}
        \frac{1}{n^2-n}\|\Pi - \hat{\Pi}\|_{F}^2 \geq \min \left\{\kappa_{4}, \kappa_3^{'} C_n\sqrt{\frac{n}{\sum_{i=1}^{n}\sum_{j=1}^{n}y_{ij}} } \right\}.
    \end{align}
    for all $n >N$. Here $\kappa_3^{'}, \kappa_4 >0 $ and $N$ are absolute constants. \rev{We set $\kappa_3 = \kappa_3^{'}/\sqrt{T}$.}
\end{theorem}
\begin{proof}
    Let $c = g^{'}(-1) = g(-1)(1-g(-1))$, and let $c^{'} = g(-1)$. Note that for all $x\in [-1,1]$, we have $g^{'}(x) \geq c$ and $c^{'} \leq g(x) \leq 1- c^{'}$. We begin by choosing $\epsilon$ so that 
\begin{align}\label{eq: epsilon sq}
    \epsilon^2 = \min\left\{\frac{c}{64},  \kappa_3^{'} C_n \sqrt{\frac{n}{\sum_{i=1}^{n}\sum_{j=1}^{n} y_{ij} }} \right\}, 
\end{align}
where $\kappa_3^{'}$ is an absolute constant to be determined. Let $k=6$ and choose $\gamma$ so that $k/\gamma^2$ is an integer and 
\begin{align*}
    4\sqrt{2} \frac{\epsilon\sqrt{2k}}{C_nc} \leq \gamma \leq \frac{8\epsilon\sqrt{2k}}{C_n c}. 
\end{align*}
This is possible since by assumption $C_n \geq \sqrt{12}$, $\epsilon \leq c/8$ and $c = 0.197$. One can check that $\gamma$ satisfies the assumptions of Lemma \ref{lm: lemma for lower bound}. Note that for $X^{(i)} \neq  X^{(j)}\in \XX$, 
\begin{align}\label{eq: lower bound in proof }
    \|g(X^{(i)}) - g(X^{(j)}) \|^{2}_{F} \geq c^2 \| X^{(i)} - X^{(j)}\|^2_{F} > c^2 C_n^2 \gamma^2 n^2/16k \geq 4\epsilon^2 n^2.
\end{align}
Now suppose for the sake of a contradiction that there exists an algorithm such that for any $X \in \KK$, it returns an $\hat{X}$ such that 
\begin{align}\label{eq: contradiction}
    \frac{1}{n^2}\|g(X) - g(\hat{X}) \|^{2}_{F} \leq \epsilon^2. 
\end{align}
with probability at least $1/4$. Define
\begin{align*}
    X^* = \argmin_{X^{(a)} \in \XX} \frac{1}{n^2} \|g(X^{(a)}) - g(\hat{X}) \|^{2}_{F}.
\end{align*}
If \eqref{eq: contradiction} holds, then \eqref{eq: lower bound in proof } implies that $X^* = X$. Thus, if \eqref{eq: contradiction} holds with probability at least $1/4$ then 
\begin{align*}
    P(X \neq X^*) \leq 3/4. 
\end{align*}
However, by a variant of Fano's inequality, we have 
\begin{align}\label{eq: Fano}
     P(X \neq X^*) \geq 1 - \frac{n^2\max_{X^{(a)} \neq X^{(b)}} D(Y \mid X^{(a)} \| Y \mid X^{(b)}) +1}{\log|\XX|}.
\end{align}
Since $y_{ij} + y_{ji}=n_{ij}$ (with $n_{ij}$ given), the value of $y_{ji}$ is determined by $y_{ij}$. Moreover, $y_{ij}$ are independent for $i = 1, \dots, n, j >i$. Therefore, 
\begin{align*}
    D(Y \mid X^{(a)} \| Y \mid X^{(b)}) = \sum_{i=1}^{n}\sum_{j>i} D(y_{ij}\mid x_{ij}^{(a)} \| y_{ij}\mid x_{ij}^{(b)}).
\end{align*}
Using Lemma \ref{lm: lemma for Dxy}, we have 
\begin{align*}
D(y_{ij} \mid x_{ij}^{(a)}  \| y_{ij} \mid x_{ij}^{(b)}) &\leq \frac{(g(C_n \gamma/\sqrt{2k})- g(-C_n \gamma/\sqrt{2k}))^2}{g(C_n \gamma/\sqrt{2k})(1-g(C_n \gamma/\sqrt{2k}) )}\\
                                                                   &\leq \frac{4(g^{'}(\xi))^2C_n^2 \gamma^2/(2k)}{g(C_n \gamma/\sqrt{2k})(1-g(C_n \gamma/\sqrt{2k}) )}\\
                                                                   & = \frac{4\{g(\xi)(1-g(\xi))\}^2C_n^2 \gamma^2/(2k)}{g(C_n \gamma/\sqrt{2k})(1-g(C_n \gamma/\sqrt{2k}) )}
\end{align*}
for some $|\xi| \leq C_n \gamma/\sqrt{2k}$. Since $c^{'} < g(x)< 1- c^{'}$ for $|x| <1$, $g(\xi) \leq g(C_n \gamma/\sqrt{2k})$, and that 
\begin{align*}
    C_n \gamma/\sqrt{2k} \leq C_n  \frac{8\epsilon\sqrt{2k}}{C_n c\sqrt{2k}} = \frac{8\epsilon}{c} \leq 1, 
\end{align*}
we have 
\begin{align*}
    D(y_{ij} \mid x_{ij}^{(a)}  \| y_{ij} \mid x_{ij}^{(b)})& \leq \frac{4(1 - c^{'})}{c^{'}}\frac{64\epsilon^2}{c^2}= \delta_4 \epsilon^2, 
\end{align*}
where $\delta_4 = 256(1 - c^{'})/(c^{'}c^2)$. Thus, from \eqref{eq: Fano}, we have 
\begin{align*}
    \frac{1}{4} \leq \frac{\delta_4 (\sum_{i=1}^{n}\sum_{j=1}^{n}y_{ij}) \epsilon^2+1 }{\log(|\XX|)} &\leq \frac{25600\gamma^2}{kn}\{\delta_4 (\sum_{i=1}^{n}\sum_{j=1}^{n}y_{ij}) \epsilon^2+1\}\\
                                                                                                     &\leq \frac{3276800}{c^2}\epsilon^2\left( \frac{\delta_4 (\sum_{i=1}^{n}\sum_{j=1}^{n}y_{ij}) \epsilon^2+1}{n C_n^2}\right).
\end{align*}
We now argue that this leads to a contradiction. Specifically, if $\delta_4 (\sum_{i=1}^{n}\sum_{j=1}^{n}y_{ij}) \epsilon^2\leq 1$, then together with \eqref{eq: epsilon sq} implies that     $ n C_n^2 \leq 409600/c$. Since $C_n^2 \geq 2k$ by assumption, if we set $N> 204800/(kc)$, this would lead to a contradiction. Thus, suppose now that $\delta_4 (\sum_{i=1}^{n}\sum_{j=1}^{n}y_{ij}) \epsilon^2 > 1$, in which case we have
$$ \epsilon^2  \geq \frac{c C_n \sqrt{n}}{5120\sqrt{\delta_4(\sum_{i=1}^{n}\sum_{j=1}^{n}y_{ij})}}.$$
Thus setting $\kappa_3^{'} \leq c/(5120\sqrt{\delta_4})$ in \eqref{eq: epsilon sq} leads to a contradiction, and hence \eqref{eq: contradiction} must fail to hold with probability at least $3/4$, which completes the proof.
\end{proof}

We now apply Theorem \ref{thm: lower bound nij given} to prove Theorem \ref{thm: lower bound}. For any $\epsilon> 0$, Hoeffding's inequality allows us to derive that
\begin{align*}
    &P\left(\sqrt{\frac{n}{\sum_{i=1}^{n}\sum_{j=1}^{n}y_{ij}} }  \geq \epsilon \sqrt{\frac{1}{np_n} }\right)\\
    =&P\left(\sum_{i=1}^{n}\sum_{j=1}^{n}y_{ij} \leq \frac{n^2p_n}{\epsilon^2}\right)\\
     =& 1- P\left(\sum_{i=1}^{n}\sum_{j>i} (n_{ij} - Tp_{ij,n}) \geq \frac{n^2p_n}{\epsilon^2} - T\sum_{i=1}^{n}\sum_{j >i}p_{ij,n} \right)\\
     \geq & 1 - P\left(\sum_{i=1}^{n}\sum_{j>i} (n_{ij} - Tp_{ij,n}) \geq \frac{n^2p_n}{\epsilon^2} - \frac{Tn(n-1)q_n}2 \right)\\
     \geq &1 - \exp\left( \frac{-2 [(n^2p_n/\epsilon^2) - \{Tn(n-1)q_n\}/2]^2}{T^2n(n-1)/2 }   \right)\\
     =& 1 - \exp\left( \frac{- \{(2n^2p_n/\epsilon^2) - Tn(n-1)q_n\}^2}{T^2n(n-1) }   \right). 
\end{align*}

To apply Theorem \ref{thm: lower bound nij given}, it suffices to find $\epsilon$ such that 
$$1 - \exp\left( \frac{- \{(2n^2p_n/\epsilon^2) - Tn(n-1)q_n\}^2}{T^2n(n-1) }   \right) \geq 0.5$$
for sufficiently large $n$. From Assumption \ref{assp: 3}, we have $p_n \asymp q_n$ and $q_n \gtrsim \log(n)/n$. Consequently, there exists $\delta_5, \delta_6 >0$ such that $p_n \geq \delta_5 q_n$ and $q_n \geq \delta_6 \log(n)/n$. Taking $\epsilon = \sqrt{\delta_5/T}$, we have 
\begin{align*}
    P\left(\sqrt{\frac{n}{\sum_{i=1}^{n}\sum_{j=1}^{n}y_{ij}} }  \geq \sqrt{\frac{\delta_5}{T}} \sqrt{\frac{1}{np_n} }\right)
    \geq &P\left(\sqrt{\frac{n}{\sum_{i=1}^{n}\sum_{j=1}^{n}y_{ij}} }  \geq \sqrt{\frac{p_n}{Tq_n}} \sqrt{\frac{1}{np_n} }\right)\\
    \geq & 1 - \exp\left(- \frac{(2Tn^2q_n - Tn(n-1)q_n)^2}{T^2n(n-1)} \right)\\
     =&1- \exp\left(-\frac{T^2n^2q_n^2(n+1)^2}{T^2n(n-1)} \right)\\
     \geq& 1 - \exp\left( -q_n^2(n+1)^2\right)\\
     \geq& 1 - \exp(-\delta_6^2 (\log(n))^2 )\\
     \geq &1/2
\end{align*}
for sufficiently large $n$. Therefore, the proof of Theorem \ref{thm: lower bound} is complete by setting $\kappa_5 = \kappa_3 \sqrt{\delta_5/T}$.

\subsection{\rev{Proof of Theorem \ref{thm: 2 to infty}}}\label{app: proves of thm: 2 to infty}
\rev{We first derive a bound for $ \|\tilde{\ThTh}- \ThTh^*\|_{F}$.
\begin{lemma}\label{lm: F consistent}
    Under Assumptions  \ref{assp: rank} and \ref{assp: p_n with rank}, for $\hat{M}$ solving \eqref{eq: nuclear and max norm estimator}, with probability converging to $1$, we have 
    \begin{align}\label{eq: M F convergence}
        \frac{1}{n^2}\|\hat{M}- M^* \|^2_{F} \leq \delta_7 C_n \sqrt{\frac{1}{p_nn}},
    \end{align}
    \begin{align}\label{eq: Theta F convergence}
        \frac{1}{n}\|\hat{\ThTh}- \ThTh^* \|^2_{F} \leq \delta_8 C_n \sqrt{\frac{1}{p_nn}}.
    \end{align}
    Here $\ThTh^* = U^*\hat{O} {\Sigma^*}^{1/2}$, where $\hat{O} = \text{diag}(\hat{O}_1, \dots, \hat{O}_K)$ is a block-diagonal matrix such that $\hat{O}_k$ is orthogonal with $\det(\hat{O}_k)=1$ for $k \in [K]$. 
\end{lemma}}
  \rev{ 
\begin{proof}
 \eqref{eq: M F convergence} follows from Theorem \ref{thm: restricted nuclear norm convergence} and the argument in the proof of Theorem 1 of \cite{davenport_etal-2014}, noting that $\|M^* \|_{\infty}$ is bounded by a positive constant from Assumption \ref{assp: rank}. Hence, we now focus on showing \eqref{eq: Theta F convergence}. 
\begin{align*}
    M^* &= \sum_{k=1}^{K} \sigma^*_k(\uu^*_{2k-1}{\uu^*_{2k}}^{\top} - \uu^*_{2k}{\uu^*_{2k-1}}^{\top} ).
\end{align*}
Note that we have $(n^2)^{-1}\|\hat{M} - M^{*}\|^2_{F} \lesssim  \delta_7 C_n (p_n n)^{-1/2}$ with probability converging to 1 by \eqref{eq: M F convergence}. Let $\hat{M}_K = \hat{U} \hat{\Sigma} J_K\hat{U}^{\top}$ denote the best rank-$2K$ approximation of $\hat{M}$ obtained by singular value decomposition. We have 
\begin{align*}
    \|\hat{M}_K - \hat{M} \|_F \leq \|M^* - \hat{M} \|_F.
\end{align*}
Therefore, we have 
\begin{align}\label{eq: rank k approx}
    \|\hat{M}_K - M^{*}\|_{F} \leq  \|\hat{M}_K - \hat{M} \|_F +\|\hat{M}  - M^*\| \leq 2\|M^* - \hat{M} \|_F
\end{align}
Using Davis-Kahan theorem, we have
\begin{align*}
    \sin (\hat{U}, U^*) \leq 2\frac{\|\hat{M}_K - M^*\|_{F}}{\sigma^*_K}. 
\end{align*}
By \eqref{eq: M F convergence}, \eqref{eq: rank k approx} and the assumption that $\sigma_K \asymp n$, we have
\begin{align*}
    \sin (\hat{U}, U^*) \lesssim   \sqrt{C_n} (p_n n)^{-1/4}
\end{align*}
with probability converging to $1$. Under this event, we have
\begin{align*}
     \|\hat{U} -  U^*\hat{O}\|_{F} \lesssim   \sqrt{C_n} (p_n n)^{-1/4}, 
\end{align*}
where $\hat{O}$ is an $2K \times 2K$ orthogonal matrix. By Theorem 2 of \cite{yu2015useful}, we can further show that for $k =1, \dots, K$,
\begin{align*}
   \sin ( (\hat{\uu}_{2k-1}, \hat{\uu}_{2k}),(\uu^*_{2k-1}, \uu^*_{2k}) ) \leq 2\frac{\|\hat{M}_K^{\top}\hat{M}_K - {M^*}^{\top}M^* \|_{F}}{ \min\{{\sigma^{*2}_{k-1}}-{\sigma^*_k}^2,{\sigma^*_k}^2 - {\sigma_{k+1}^{*2}} \} },
\end{align*}
where $\sigma_{K+1} = 0$ by definition. 
Define $E = \hat{M}_K -M^* $, we have 
\begin{align*}
    \|\hat{M}_K^{\top}\hat{M}_K - {M^*}^{\top}M^*\|_{F} =  \|{M^*}^{\top}E + E^{\top}M^* + {E}^{\top}E\|_{F} \leq 2\|M^*\|_{op}\|E\|_{F} + \|E\|^2_{F}
\end{align*}and thus $\|(\hat{\uu}_{2k-1}, \hat{\uu}_{2k}) - (\uu^*_{2k-1}, \uu^*_{2k}) \hat{O}_k\|_{F} \lesssim   \sqrt{C_n} (p_n n)^{-1/4} $ with probability converging to $1$ , where $\hat{O}_{2k}$ is a $2 \times 2 $ orthogonal matrix. In this event, we have 
\begin{align*}
    \|\hat{U} - U^* \text{diag}(\hat{O}_1, \dots, \hat{O}_K) \|_{F}  \lesssim   \sqrt{C_n} (p_n n)^{-1/4} . 
\end{align*}
We now show that $\det(O_k) =1$ for $k = 1, \dots, K$.
We first show it when $K=1$. Note that $\det(O_1)= 1$ or $-1$ as $O_{1}$ is orthogonal.  If $\det(O_1) = -1$, we have
 \begin{align*}
U^*\hat{O}_1 \Sigma^*J_1 (U^*\hat{O}_1)^{\top} &=  U^*\hat{O}_1\Sigma^* J_1  \hat{O}_1^{\top}  {U^*}^{\top} \\
&= \sigma_1^*\det(\hat{O}_1) U^* J_1 {U^*}^{\top} \\
&=\det(\hat{O}_1) U^*\Sigma^* J_1 {U^*}^{\top} \\
&= \det(\hat{O}_1) M^* \\
&= -M^*.
 \end{align*}
 However this implies that we can write $n^{-1}\|M^*+\hat{M}\|_{F} = n^{-1}\|-M^*-\hat{M}\|_{F}$ as 
 \begin{align*}
    &\frac{1}{n} \|  U^*\hat{O}_1 \Sigma^* J_1(U^*\hat{O}_1 )^{\top} - \hat{U} \hat{\Sigma}J_1(\hat{U})^{\top}\|_{F}\\
    \leq& \frac{1}{n}\| U^*\hat{O}_1 \Sigma^*J_1  (U^*\hat{O}_1 - \hat{U})^{\top}\|_{F} + \frac{1}{n}\|( U^*\hat{O}_1 \Sigma^* - \hat{U} \hat{\Sigma})J_1(\hat{U})^{\top} \|_{F}\\
    \leq& \frac{1}{n}\|\Sigma^* \|_{2}\|U^*\hat{O}_1 - \hat{U}\|_{F} +\frac{1}{n} \|(U^*\hat{O}_1 -\hat{U})\Sigma^*  \|_{F} + \frac{1}{n}\|\hat{U}( \Sigma^* - \hat{\Sigma})\|_{F}\\
    \leq& \frac{2\sigma_1^*}{n}\|U^*\hat{O}_1 - \hat{U}\|_{F} + \frac{1}{n}\|\Sigma^* - \hat{\Sigma}\|_{F}\\
    \lesssim&  \sqrt{C_n} (p_n n)^{-1/4}.
 \end{align*}
This implies that $n^{-1}\|2M^*\|_{F} \lesssim n^{-1} \|M^* - \hat{M}\|_{F} + \|M^* + \hat{M}\|_{F} \lesssim  \sqrt{C_n} (p_n n)^{-1/4}$. However, it leads to contradiction as $n^{-1}\|2M^*\|_{F} \geq n^{-1}\|M^*\|_2 = \sigma_1/n$, which does not vanishes. Hence we are forced to have $\det(O_k) = 1$. 
The argument for general $K$ is similar. Specifically, suppose there are at least one $k \in[K]$ such that $\det(O_k) = -1$. Let $U^*_{sub}$ be the matrix containing the columns of $U^*$ corresponding to $\sigma^*_k$ such that $\det(O_k) = -1$, and similarly let $\Sigma^*_{sub}$ be  the diagonal matrix containing those pairs of $\sigma_k^*$. Define $M^*_{sub} = U_{sub}^*\Sigma_{sub}^* J_{K_{sub}}(U^* )^{\top} $, where $K_{sub}$ is the number of k with  $\det(O_k) = -1$. We can then show that $n^{-1}\|M^*_{sub}+\hat{M}_{sub}\|_{F}  \lesssim \sqrt{C_n} (p_n n)^{-1/4} $ and $n^{-1}\|M^*_{sub} - \hat{M}_{sub}\|_{F}\lesssim  \sqrt{C_n} (p_n n)^{-1/4} $, which leads to contradiction.

Recall that we defined $\hat{\ThTh} = P_{\tau_2}^{2,\infty}(\hat{U} {\hat{\Sigma}}^{1/2})$ and $\ThTh^* =U^*\hat{O} {\Sigma^*}^{1/2}$. Take $\tau_2 = C^*$. Since $\|\ThTh^*\|_{2 \to \infty} \leq C^*$ by assumption, we have
\begin{align*}
     &n^{-1/2}\|\hat{\ThTh}- \ThTh^* \|_{F} \\
     =& n^{-1/2}\|P_{C^*}^{2,\infty}(\hat{U} {\hat{\Sigma}}^{1/2}) - U^* \hat{O}{\Sigma^*}^{1/2}\|_{F}\\
     \leq& n^{-1/2} \|P_{C^*}^{2, \infty}(\hat{U} {\hat{\Sigma}}^{1/2}) -\hat{U} {\hat{\Sigma}}^{1/2}\|_F +\| \hat{U} {\hat{\Sigma}}^{1/2}-   U^*\hat{O}{\Sigma^*}^{1/2}\|_{F}\\
                                                                \leq& 2n^{-1/2}\| \hat{U} {\hat{\Sigma}}^{1/2}-   U^*\hat{O}{\Sigma^*}^{1/2}\|_{F}\\
                                                                 \leq& n^{-1/2}\|\hat{U}(\hat{\Sigma}^{1/2}- {\Sigma^*}^{1/2}) \|_{F} + n^{-1}\|(\hat{U} - U^*\hat{O}){\Sigma^*}^{1/2}  \|_{F} \\
    \leq& n^{-1/2}\| \hat{\Sigma}^{1/2}- {\Sigma^*}^{1/2}\|_{F} + n^{-1/2}\sqrt{\sigma_1}\|\hat{U} - U^* \hat{O} \|_{F}\\
    \lesssim&  n^{-1/2}\|\hat{M}  - M^* \|_{F} + \|\hat{U} - U^* \hat{O} \|_{F}\\
    \lesssim&  \sqrt{C_n} (p_n n)^{-1/4},
\end{align*}
which completes the proof of Lemma \ref{lm: F consistent}. 
\end{proof}
}

\rev{Let $\omega_{ij} = I(n_{ij} \geq 1)$ be the missing indicator. Define the functions $\alpha_1(u)=\sup_{|x|\leq u} |g^{'}(u)|$ and  $\alpha_2(u)=\sup_{|x|\leq u}| g^{''}(u)|$. The following Lemma provides a non-probabilistic bound for the solution to $S_{1,i}(\thth;\ThTh_{-i})=0$.
\begin{lemma}
    Let $\text{diag}(\Omega_i) = \text{diag}(\{\omega_{ij}\}_{j \in [n], j \neq i})$ be a diagonal matrix of $\omega_{ij}$ for $j = 1, \dots n, j \neq i$, and $\zz_i =(z_{ij})_{j \in [n], j \neq i},$ where $z_{ij} = y_{ij} - n_{ij}g(m_{ij}^*)$. Under Assumptions \ref{assp: rank} and \ref{assp: p_n with rank}, if there exists $\xi>0$ such that 
    \begin{align*}
        &2\sigma_{2K}^{-1}(\CI_{1,i}(\ThTh_{-i}) ) \{  \|\zz_i  \text{diag}(\Omega_i){\ThTh_{-i}^{(p)}} \| +\|\CB_{1,i}(\ThTh_{-i})\| + \beta_{1,i}(\ThTh_{-i})\alpha_2 \left( (C^* + \xi)C^* \right) \} \\
        \leq& \xi\leq 0.5 \{\gamma_{1,i}(\ThTh_{-i})\alpha_2 \left( (C^* + \xi)C^* \right)\}^{-1}\sigma_{2K}(\CI_{1,i}(\ThTh_{-i}) ),
    \end{align*} where we define
    \begin{align*}
        \CB_{1,i}(\ThTh_{-i})& =\sum_{j \in [n], j \neq i}  n_{ij}g^{'}(m_{ij}^{*})\thth_j^{(p)}(\thth_j^{(p)} - {\thth_j^{*(p)}} )^{\top}\thth_i^*,\\
        \CI_{1,i}(\ThTh_{-i}) &= \sum_{j \in [n], j \neq i}  n_{ij}g^{'}(m_{ij}^{*})\thth_j^{(p)}(\thth_j^{(p)})^{\top}, \\
        \beta_{1,i}(\ThTh_{-i}) &=  \sup_{\|u\|=1}\sum_{j \in [n], j \neq i}  n_{ij}((\thth_j^{(p)} - {\thth_j^{*(p)}} )^{\top}\thth_i^*)^2 |{\thth_j^{(p)}}^{\top}\uu|, \text{ and }\\
        \gamma_{1,i}(\ThTh_{-i})&=\sup_{\|u\|=1}\sum_{j \in [n], j \neq i}  n_{ij}|{\thth_j^{(p)}}^{\top}\uu|^3. 
    \end{align*}
    Then, there is $\breve{\thth}_i$ such that $\|\breve{\thth}_i- \thth_i^* \| \leq \xi$ and $S_{1,i}(\breve{\thth}_i;\ThTh_{-i})=0 $. 
\end{lemma}}
\rev{
\begin{proof}
 Let $\thth$ be a vector such that $\|\thth - \thth^*_i\| \leq \xi$ and let $m_{ij} = \thth^{\top}\thth_j^{(p)}$. Consider the Taylor expansion of  $S_{1,i}(\thth;\ThTh_{-i})$: 
\begin{align*}
    S_{1,i}(\thth;\ThTh_{-i}) =& \sum_{j \in [n], j \neq i}  \omega_{ij}\left\{ y_{ij} - n_{ij}g(m_{ij}^* )   \right\}\thth_j^{(p)} - \sum_{j \in [n], j \neq i}  n_{ij}\left\{g(m_{ij}^* ) - g(m_{ij} )   \right\}\thth_j^{(p)} \\
    =& {\ThTh_{-i}^{(p)}}^{\top} \text{diag}(\Omega_i) \zz_i^{\top} - \sum_{j \in [n], j \neq i}  n_{ij}g^{'}(m_{ij}^{*})(m_{ij} - m_{ij}^*)\thth_j^{(p)} \\
    &- 0.5 \sum_{j \in [n], j \neq i}  n_{ij}g^{''}(\tilde{m}_{ij})(m_{ij} - m_{ij}^*)^2 \thth_j^{(p)}.
\end{align*}
for some $\tilde{m}_{ij}$ between $m_{ij}^*$ and $m_{ij}$. Plugging $m_{ij} - m_{ij}^* = (\thth_j^{(p)})^{\top}(\thth - \thth_i^*) + (\thth_j^{(p)} - {\thth_j^{*(p)}} )^{\top}\thth_i^*$ into the above display, we obtain 
\begin{align*}
     S_{1,i}(\thth;\ThTh_{-i})   =& {\ThTh_{-i}^{(p)}}^{\top} \text{diag}(\Omega_i) \zz_i^{\top} - \sum_{j \in [n], j \neq i}  n_{ij}g^{'}(m_{ij}^{*})\thth_j^{(p)}(\thth_j^{(p)})^{\top}(\thth - \thth_i^*)  \\
    &- \sum_{j \in [n], j \neq i}  n_{ij}g^{'}(m_{ij}^{*})\thth_j^{(p)}(\thth_j^{(p)} - {\thth_j^{*(p)}} )^{\top}\thth_i^* \\
    &- 0.5 \sum_{j \in [n], j \neq i}  n_{ij}g^{''}(\tilde{m}_{ij})(m_{ij} - m_{ij}^*)^2 \thth_j^{(p)}. 
\end{align*}
Multiplying $(\thth - \thth_i^*)^{\top}$ on both sides, we obtain 
\begin{align}\label{eq: theta S expansion}
    &(\thth - \thth_i^*)^{\top}S_{1,i}(\thth;\ThTh_{-i})\nonumber \\
    =& (\thth - \thth_i^*)^{\top}{\ThTh_{-i}^{(p)}}^{\top} \text{diag}(\Omega_i) \zz_i^{\top} - (\thth - \thth_i^*)^{\top}\sum_{j \in [n], j \neq i}  n_{ij}g^{'}(m_{ij}^{*})\thth_j^{(p)}(\thth_j^{(p)})^{\top}(\thth - \thth_i^*)\nonumber \\
    &- (\thth - \thth_i^*)^{\top}\sum_{j \in [n], j \neq i}  n_{ij}g^{'}(m_{ij}^{*})\thth_j^{(p)}(\thth_j^{(p)} - {\thth_j^{*(p)}} )^{\top}\thth_i^* \nonumber\\
    &- 0.5(\thth - \thth_i^*)^{\top} \sum_{j \in [n], j \neq i}  n_{ij}g^{''}(\tilde{m}_{ij})(m_{ij} - m_{ij}^*)^2 \thth_j^{(p)}. 
\end{align}
Recall that $\|\thth - \thth_i^*\| = \xi$. Using inequalities about matrix products and singular values, we have the following upper bounds for the first three terms on the right-hand side of the above display/ 
\begin{align}\label{eq: 1}
    |(\thth - \thth_i^*)^{\top}{\ThTh^{(p)}}^{\top} \text{diag}(\Omega_i) \zz_i^{\top}| \leq |\xi|\|{\ThTh^{(p)}}^{\top} \text{diag}(\Omega_i) \zz_i^{\top}\| = \xi \|\zz_i  \text{diag}(\Omega_i){\ThTh_{-i}^{(p)}} \|,  \\
    - (\thth - \thth_i^*)^{\top}\sum_{j \in [n], j \neq i}  n_{ij}g^{'}(m_{ij}^{*})\thth_j^{(p)}(\thth_j^{(p)})^{\top}(\thth - \thth_i^*) \leq - \xi^2 \sigma_{2K}(\CI_{1,i}(\ThTh_{-i}) ),
    \end{align}
    where $\sigma_{2K}(\CI_{1,i}(\ThTh_{-i}) )$ denotes the $2K$th largest singular value of $\CI_{1,i}(\ThTh_{-i})$, and 
    \begin{align}\label{eq: 2}
        |(\thth - \thth_i^*)^{\top}\sum_{j \in [n], j \neq i}  n_{ij}g^{'}(m_{ij}^{*})\thth_j^{(p)}(\thth_j^{(p)} - {\thth_j^{*(p)}} )^{\top}\thth_i^* |= \|(\thth - \thth_i^*)^{\top}\CB_{1,i}(\ThTh_{-i})\| \leq \xi \|\CB_{1,i}(\ThTh_{-i})\|.
    \end{align}
    Now we analyze the last term in \eqref{eq: theta S expansion}. Note that $|\tilde{m}_{ij}| \leq |m_{ij}^*| \vee |m_{ij}| \leq (C^* + \xi)C^*$ and $m_{ij} - m_{ij}^* = (\thth_j^{(p)})^{\top}(\thth - \thth_i^*) + (\thth_j^{(p)} - {\thth_j^{*(p)}} )^{\top}\thth_i^*$. We have 
    \begin{align*}
        &0.5(\thth - \thth_i^*)^{\top}\sum_{j \in [n], j \neq i}  n_{ij}g^{''}(\tilde{m}_{ij})(m_{ij} - m_{ij}^*)^2 \thth_j^{(p)} \\
        \leq&0.5\alpha_2 \left( (C^* + \xi)C^* \right) \xi \sup_{\|u\|=1}\sum_{j \in [n], j \neq i}  n_{ij}((\thth_j^{(p)} - {\thth_j^{*(p)}} )^{\top}\thth_i^*+ \xi(\thth_j^{(p)})^{\top}u  )^2 |{\thth_j^{(p)}}^{\top}\uu|\\
        \leq &\alpha_2 \left( (C^* + \xi)C^* \right) \bigg\{ \xi \sup_{\|u\|=1}\sum_{j \in [n], j \neq i}  n_{ij}((\thth_j^{(p)} - {\thth_j^{*(p)}} )^{\top}\thth_i^*)^2 |{\thth_j^{(p)}}^{\top}\uu| \\
        &\quad \quad \quad \quad \quad \quad \quad \quad +\xi^3 \sup_{\|u\|=1}\sum_{j \in [n], j \neq i}  n_{ij}|{\thth_j^{(p)}}^{\top}\uu|^3 \bigg\}\\
        =&\alpha_2 \left( (C^* + \xi)C^* \right)(\xi \beta_{1,i}(\ThTh_{-i})+\xi^3 \gamma_{1,i}(\ThTh_{-i})).
    \end{align*}
    Combining the analysis with \eqref{eq: 1} and \eqref{eq: 2}, we obtain 
    \begin{align*}
         (\thth - \thth_i^*)^{\top}S_{1,i}(\thth;\ThTh_{-i}) \leq& -  \sigma_{2K}(\CI_{1,i}(\ThTh_{-i}) )\xi^2 + \gamma_{1,i}(\ThTh_{-i})\alpha_2 \left( (C^* + \xi)C^* \right)\xi^3 \\
         &  +  \{  \|\zz_i  \text{diag}(\Omega_i){\ThTh_{-i}^{(p)}} \| +\|\CB_{1,i}(\ThTh_{-i})\| + \beta_{1,i}(\ThTh_{-i})\alpha_2 \left( (C^* + \xi)C^* \right) \}\xi. 
    \end{align*}
    The rest of the proof follows from the argument in the proof of Lemma 15 in \cite{chen_Li-2024-JoMR}. 
\end{proof}}

    \rev{We now simplify the result of Lemma 15 to a more user-friendly version. 
    \begin{lemma}\label{lm: lemma 16 in chen jmlr}
        Under Assumptions \ref{assp: rank} and \ref{assp: p_n with rank}, if 
        \begin{align*}
            \|\zz_i  \text{diag}(\Omega_i){\ThTh_{-i}^{(p)}} \| +\|\CB_{1,i}(\ThTh_{-i})\| + \beta_{1,i}(\ThTh_{-i})\alpha_2 \left( {3C^*}^3\right) \\
            \leq \min\{ 0.25 \gamma_{1,i}(\ThTh_{-i})^{-1}(\alpha_2 \left( {3C^*}^3\right))^{-1} \sigma^2_{2K}(\CI_{1,i}(\ThTh_{-i}) ), 0.5 \sigma_{2K}(\CI_{1,i}(\ThTh_{-i}) )C^*  \}.
        \end{align*}
        Then, there is $\breve{\thth}_i$ such that $S_{1,i}(\breve{\thth}_i; \ThTh_{-i})=0$, and 
        \begin{align*}
            \|\breve{\thth}_i - \thth_i^*\| \leq 2\sigma_{2K}^{-1}(\CI_{1,i}(\ThTh_{-i}) )\{ \|\zz_i  \text{diag}(\Omega_i){\ThTh_{-i}^{(p)}} \| +\|\CB_{1,i}(\ThTh_{-i})\| + \beta_{1,i}(\ThTh_{-i})\alpha_2 \left( {3C^*}^3\right)\}.
        \end{align*}
    \end{lemma}
    The proof follows from the argument in Lemma 16 of \cite{chen_Li-2024-JoMR}. We omit the details. We now analyze each term in Lemma \ref{lm: lemma 16 in chen jmlr} with $\ThTh_{-i}   = \hat{\ThTh}_{-i}$. Let $p_{\max} = \max_{i \in [n]}\sum_{j \in [n], j \neq i}n_{ij}$ denotes the maximum number of comparisons in each row.}
    \rev{\begin{lemma}\label{lm: Lemma 43 in chen et al JMLR}
        (Upper bound for $\|\zz_i  \text{diag}(\Omega_i) \hat{\ThTh}_{-i}^{(p)} \|$). Under Assumptions \ref{assp: rank} and \ref{assp: p_n with rank}, with probability at least $ 1 - (nK)^{-1}$, 
        \begin{align*}
            \max_{i \in [n]} \|\zz_i  \text{diag}(\Omega_i){\hat{\ThTh}_{-i}^{(p)}}\| \leq 16 C^* \sqrt{2K} \log(nK)p^{1/2}_{\max} + \sqrt{T} p_{\max}^{1/2}\|\hat{\ThTh}^{(p)} -\ThTh^{*(p)} \|_{F}. 
        \end{align*}
    \end{lemma}}
    \rev{
    \begin{proof}
    Note that 
    \begin{align}\label{eq: z omega th}
         \|\zz_i  \text{diag}(\Omega_i){\hat{\ThTh}_{-i}^{(p)}}\| \leq  \|\zz_i  \text{diag}(\Omega_i){\ThTh^{*(p)}_{-i}} \| + \|\zz_i  \text{diag}(\Omega_i)({\hat{\ThTh}_{-i}^{(p)}} -\ThTh^{*(p)}_{-i})\|. 
    \end{align}
    It is easy to see that $|y_{ij} - n_{ij}g(m_{ij}^*)|\leq n_{ij}$ for all $i,j$. Hence we have 
    \begin{align}\label{eq: z omega th minus}
        \|\zz_i  \text{diag}(\Omega_i)({\hat{\ThTh}_{-i}^{(p)}} -\ThTh^{*(p)}_{-i})\|  &= \|\sum_{j \in [n], j \neq i}  (y_{ij} - n_{ij}g(m_{ij}^*))(\hat{\thth}_j - \thth_j^*) \|\nonumber \\
                                                                                       &\leq \sum_{j \in [n], j \neq i}  n_{ij}\|(\hat{\thth}_j - \thth_j^*)  \|\nonumber\\
                                                                                       &\leq \sqrt{\sum_{j \in [n], j \neq i}  n_{ij}^2} \|\hat{\ThTh}_{-i}^{(p)} -\ThTh^{*(p)}_{-i} \|_{F}\nonumber\\
                                                                                       &\leq \sqrt{T} p_{\max}^{1/2}\|\hat{\ThTh}_{-i}^{(p)} -\ThTh^{*(p)}_{-i} \|_{F}\nonumber\\
                                                                                       &\leq \sqrt{T} p_{\max}^{1/2}\|\hat{\ThTh}^{(p)} -\ThTh^{*(p)} \|_{F}.
    \end{align}
    Combining \eqref{eq: z omega th} and \eqref{eq: z omega th minus} and taking maximum over $i \in [n]$, we have 
    \begin{align*}
        \max_{i \in [n]}  \|\zz_i  \text{diag}(\Omega_i){\hat{\ThTh}_{-i}^{(p)}}\| \leq  \max_{i \in [n]} \|\zz_i  \text{diag}(\Omega_i){\ThTh^{*(p)}_{-i}} \|  + \sqrt{T} p_{\max}^{1/2}\|\hat{\ThTh}^{(p)} -\ThTh^{*(p)} \|_{F}. 
    \end{align*}
    For the first term, we can show that for each $k \in \{1, \dots, 2K\}$ and positive $t$, since $y_{ij} - n_{ij}g(m_{ij}^*)$ is mean zero given $n_{ij}$, we can apply Bernstein inequality and get 
    \begin{align*}
        &P\bigg(\bigg|\sum_{j \in [n], j \neq i}  (y_{ij} - n_{ij}g(m_{ij}^*))\theta_{jk}^{*(p)})\bigg|\geq t  \mid \{n_{ij}\}_{j \in \in [n], j \neq i}\bigg) \\
        \leq& \exp\left(  -\frac{0.5t^2}{C^*(\sum_{j \in [n], j \neq i} n_{ij}g(m_{ij}^*)(1-g(m_{ij}^*)) + \frac{1}{3}Tt) }  \right)\\
        \leq& \exp\left(  -\frac{0.5t^2}{C^*(\sum_{j \in [n], j \neq i} n_{ij} + Tt) }  \right). 
    \end{align*}
    This implies that 
    \begin{align*}
         &P(\|\zz_i  \text{diag}(\Omega_i){\ThTh^{*(p)}_{-i}} \|\geq t  \mid \{n_{ij}\}_{j \in \in [n], j \neq i}) \\
         \leq& \sum_{k=1}^{2K} P(|\sum_{j \in [n], j \neq i}  (y_{ij} - n_{ij}g(m_{ij}^*))\theta_{jk}^{*(p)})|\geq t/\sqrt{2K}  \mid \{n_{ij}\}_{j \in \in [n], j \neq i}) \\
         \leq& 2K \exp\left(  -\frac{0.25 t^2}{KC^*(\sum_{j \in [n], j \neq i} n_{ij} + Tt/\sqrt{2K}) }  \right).
    \end{align*}
    Combining results for different $i$ with a union bound, we have 
    $$ P(  \max_{i \in [n]} \|\zz_i  \text{diag}(\Omega_i){\ThTh^{*(p)}_{-i}} \|\geq t  \mid \{n_{ij}\}_{j \in \in [n], j \neq i}) \leq 2Kn \exp\left(  -\frac{0.25 t^2}{KC^*(p_{\max} + Tt/\sqrt{2K}) }  \right).$$
    In particular, for $t = 16 C^* \sqrt{2K} \log(nK)p^{1/2}_{\max}  $, we can show that $\max_{i \in [n]} \|\zz_i  \text{diag}(\Omega_i){\ThTh^{*(p)}_{-i}} \| \leq 16 C^* \sqrt{2K} \log(nK)p^{1/2}_{\max}$ 
    with probability at least $1  - (nK)^{-1}$, and the proof is complete.
        \end{proof}}

    \rev{\begin{lemma}\label{lm: Bound of CB1i}
        (Upper bound for $\| \CB_{1,i}(\hat{\ThTh}_{-i})\|$) Under Assumptions \ref{assp: rank} and \ref{assp: p_n with rank},  
        \begin{align*}
           \max_{i \in [n]} \| \CB_{1,i}(\hat{\ThTh}_{-i})\| \leq {C^*}^2 \sqrt{T}p_{\max}^{1/2}\|\hat{\ThTh} - \ThTh^* \|_{F}.
        \end{align*}
    \end{lemma}}
    \rev{
    \begin{proof}
    Recall that  $|g^{'}(x)| = |g(x)(1-g(x))| \leq 1$. Hence we have 
    \begin{align*}
       \| \CB_{1,i}(\hat{\ThTh}_{-i})\|&=\|\sum_{j \in [n], j \neq i}  n_{ij}g^{'}(m_{ij}^*)\hat{\thth}_j^{(p)}(\hat{\thth}_j^{(p)} - {\thth_j^{*(p)}} )^{\top}\thth_i^*\|\\
                                 &\leq {C^*}^2 \sum_{j \in [n], j \neq i} n_{ij}\|\hat{\thth}_j^{(p)} - {\thth_j^{*(p)}}\|\\
                                 &\leq  {C^*}^2 \sqrt{T}p_{\max}^{1/2}\|\hat{\ThTh}_{-i} - \ThTh^*_{-i} \|\\
                                 &\leq {C^*}^2 \sqrt{T}p_{\max}^{1/2}\|\hat{\ThTh} - \ThTh^* \|.
    \end{align*}
    The proof is completed by taking maximum for $i \in [n]$.
       \end{proof}}
\rev{
    \begin{lemma}\label{lm: bound for beta1i}
        (Bound for $\beta_{1,i}(\hat{\ThTh}_{-i})$) Under Assumptions \ref{assp: rank} and \ref{assp: p_n with rank},
        \begin{align*}
            \max_{i \in [n]} \beta_{1,i}(\hat{\ThTh}_{-i}) \leq {C^{*3}} T\|\hat{\ThTh}  - \ThTh^*\|^2_F. 
        \end{align*}
    \end{lemma}
    \begin{proof}
     For each $i \in [n]$
    \begin{align*}
        \beta_{1,i}(\hat{\ThTh}_{-i}) &=  \sup_{\|u\|=1}\sum_{j \in [n], j \neq i}  n_{ij}((\hat{\thth}_j^{(p)} - {\thth_j^{*(p)}} )^{\top}\thth_i^*)^2 |(\hat{\thth}_j^{(p)})^{\top}\uu| \\
        &\leq {C^{*3}} \sup_{\|u\|=1}\sum_{j \in [n], j \neq i}  n_{ij}\|\hat{\thth}_j^{(p)} - {\thth_j^{*(p)}} \|^2\leq {C^{*3}} T\|\hat{\ThTh}  - \ThTh^*\|^2_F
    \end{align*}
    \end{proof}}

\rev{
    \begin{lemma}\label{lm: lemma for pmax}
        (Upper bound for $p_{\max}$). Under Assumptions \ref{assp: rank} and \ref{assp: p_n with rank}, if $n q_n \geq 6\log(n)$, then 
        \begin{align*}
            P(p_{\max} \geq 2Tn q_n  ) \leq 1/n. 
        \end{align*}
    \end{lemma}
    \begin{proof}
    It follows from Lemma 23 in \cite{chen_Li-2024-JoMR} that $P(\max_{i\in [n]}\sum_{j\in [n], j \neq i}\omega_{ij} \geq 2pq_n ) \leq 1/n$, which implies  $P(\max_{i\in [n]}\sum_{j\in [n], j \neq i}T\omega_{ij} \geq 2Tpq_n ) \leq 1/n$. The proof then complete as $P(p_{\max} \geq 2Tpq_n )\leq P(\max_{i\in [n]}\sum_{j\in [n], j \neq i}T\omega_{ij} \geq 2Tpq_n ). $ \end{proof}}
    \rev{
    \begin{lemma}\label{lm: bound for gamma1i}
         (Bound for $\gamma_{1,i}(\hat{\ThTh}_{-i})$). Under Assumptions \ref{assp: rank} and \ref{assp: p_n with rank}, if $n q_n \geq 6\log(n)$, then with probability at least $1  - 1/n$,
         \begin{align*}
             \gamma_{1,i}(\hat{\ThTh}_{-i}) \leq 2Tn q_n {C^{*3}}.
         \end{align*}
    \end{lemma}
    \begin{proof}
     The lemma follow by Lemma \ref{lm: lemma for pmax} and the following inequality
    \begin{align*}
          \gamma_{1,i}(\ThTh_{-i})=\sup_{\|u\|=1}\sum_{j \in [n], j \neq i}  n_{ij}|{\thth_j^{(p)}}^{\top}\uu|^3 \leq {C^{*3}} p_{\max}
    \end{align*}
      \end{proof}
       The next two lemmas give a lower bound for $\sigma_{2K}(\CI_{1,i}(\hat{\ThTh}_{-i}))$. }
    \rev{
    \begin{lemma}\label{lm: lower bound of I final}
        Define $\mu(y) = \inf_{|x| \leq y} g^{'}(x)$. Under Assumptions \ref{assp: rank} and \ref{assp: p_n with rank}, if $\|\text{diag}(\Omega_i)({\hat{\ThTh}_{-i}^{(p)}} - {\ThTh_{-i}^{*(p)}})\|_{2} \leq 2^{-1}\sigma_{2K}( \text{diag}(\Omega_i){\ThTh_{-i}^{*(p)}} )$, then 
        \begin{align*}
            \sigma_{2K}(\CI_{1,i}(\hat{\ThTh}_{-i})) \geq 2^{-2}\mu({C^*}^2)\sigma_{2K}^2( \text{diag}(\Omega_i){\ThTh_{-i}^{(*p)}} ).
        \end{align*}
    \end{lemma}}
    \rev{
    \begin{proof}
        Note that for any$\uu \in \RR^{2K}$ such that $\|\uu\|=1$, 
    \begin{align*}
        \CI_{1,i}(\hat{\ThTh}_{-i}) &= \sum_{j \in [n], j \neq i}  n_{ij}g^{'}(m_{ij}^{*}) (\uu^{\top}\hat{\thth}_j^{(p)})^2\\
                                    &\geq \mu({C^*}^2)\sum_{j \in [n]}  \omega_{ij}(\uu^{\top}\hat{\thth}_j^{(p)})^2\\
                                    &\geq  \mu({C^*}^2) \sigma_{2K}^2( \text{diag}(\Omega_i){\hat{\ThTh}_{-i}^{(p)}} ). 
    \end{align*}
    The rest of the proof follows from the argument in Lemma 25 of \cite{chen_Li-2024-JoMR}. \end{proof}}
        \rev{
    \begin{lemma}\label{lm: lower bound of I}
        Let $P_{1,i} = \text{diag}(\{p_{ij}\}_{j \in [n], j \neq i}) = E(\text{diag}(\Omega_i))$ and $\lambda^*_{i,\min} = \lambda_{2K}((\ThTh_{-i}^*)^{\top}P_{1,i} \ThTh_{-i}^* )$, where $\lambda_{2K}(\cdot)$ denotes the $2K$th largest eigenvalue of a symmetric matrix. If $\lambda^*_{\min} \coloneq \min_{i \in [n]}\lambda^*_{i,\min} \geq 16 \|\ThTh^*\|_{2 \to \infty}^2\log(2Kn)$, then 
        \begin{align*}
            P(\min_{i \in [n]} \sigma_{2K}^2( \text{diag}(\Omega_i){\ThTh_{-i}^{*(p)}} )\leq 0.5 \lambda^*_{\min}) \leq 1/(2nK).
        \end{align*}
        Moreover, if $p_n\sigma_{2K}^2(\ThTh^*) \geq 32 \|\ThTh^*\|^2_{2 \to \infty}\log(n)$, then 
        \begin{align*}
              P(\min_{i \in [n]} \sigma_{2K}^2( \text{diag}(\Omega_i){\ThTh_{-i}^{*(p)}} )\leq 0.5 p_n \sigma_{2K}^2(\ThTh^*) ) \leq 1/(2nK).
        \end{align*}
    \end{lemma}
          }
    \rev{The proof follows from the argument in the proof of Lemma 26 in \cite{chen_Li-2024-JoMR}, noting that $\text{diag}(\Omega_i){\ThTh_{-i}^{*(p)}} = \text{diag}(\bar\Omega_i){\ThTh^{*(p)}}$, where $\bar{\Omega}_i = \text{diag}(\omega_{i1}, \dots, \omega_{in})$. We omit the details. }
    \rev{We now turn to asymptotic analysis of $\tilde{\ThTh}$. 
    \begin{lemma}\label{lm: ThTh hat 2 to inf convergence}
        Under Assumptions \ref{assp: rank} and \ref{assp: p_n with rank}, with probability converging to $1$, there is $\tilde{\ThTh} = (\tilde{\theta}_{ik})_{n \times 2K}$ such that $S_{1,i}(\tilde{\thth}_i, \hat{\ThTh}_{-i})= 0$ for all $i \in [n]$, $\|\tilde{\ThTh} - \ThTh^* \|_{2\to \infty }\leq C^*$, and 
        \begin{align*}
            \|\tilde{\ThTh} - \ThTh^*\|_{ 2\to \infty} \lesssim  n^{1/2}(nq_n)^{1/4} (np_n)^{-1}  = n^{-1/4}q_n^{-3/4}. 
        \end{align*}
        Moreover, $\tilde{\thth}_i$ is the unique solution to the optimization problem $\max_{\thth_i \in \RR^{2K}} l_{i}(\thth_i, \hat{\ThTh}_{-i}) $ for all $ i \in [n]$.
    \end{lemma}}
    \rev{
    \begin{proof}
        Throughout the proof, we restrict the analysis on the event  $\{p_{\max} \leq 2Tnq_n\}$, which has probability converging to $1$ by Lemma \ref{lm: lemma for pmax}. On this event, we have that with probability at least $1- 1/(nK)^{-1}$, 
    \begin{align}\label{eq: max 1}
         \max_{i \in [n]} \|\zz_i  \text{diag}(\Omega_i){\hat{\ThTh}_{-i}^{(p)}}\| \leq 32 C^* K \log(nK)(Tnq_n)^{1/2} + \sqrt{2T} (Tnq_n)^{1/2}\|\hat{\ThTh} - \ThTh^* \|_{F}
    \end{align}
    according to Lemma \ref{lm: Lemma 43 in chen et al JMLR} and the fact that $\|\hat{\ThTh}^{(p)} - \ThTh^{*(p)} \|_{F} = \|\hat{\ThTh} - \ThTh^* \|_{F}$. Next, according to Lemma \ref{lm: Bound of CB1i}, we have 
    \begin{align*}
         \max_{i \in [n]} \| \CB_{1,i}(\hat{\ThTh}_{-i})\| \leq {C^*}^2 \sqrt{T}p_{\max}^{1/2}\|\hat{\ThTh} - \ThTh^* \|_{F}
    \end{align*}
    Thus with probability converging to one, we have 
    \begin{align}\label{eq: max 2}
         \max_{i \in [n]} \| \CB_{1,i}(\hat{\ThTh}_{-i})\| \leq {C^*}^2 \sqrt{T}\sqrt{2Tnq_n}\|\hat{\ThTh} - \ThTh^* \|_{F}.
    \end{align}
    Next, according to Lemma \ref{lm: bound for beta1i}, we have 
    \begin{align}\label{eq: max 3}
          \max_{i \in [n]} \beta_{1,i}(\hat{\ThTh}_{-i}) \leq {C^{*3}} T \|\hat{\ThTh} - \ThTh^* \|_{F}^2.
    \end{align}
    Combining \eqref{eq: Theta F convergence}, \eqref{eq: max 1}, \eqref{eq: max 2} and \eqref{eq: max 3}, we have 
    \begin{align}\label{eq: max z B beta}
        &\max_{i \in [n]} \{\|\zz_i  \text{diag}(\Omega_i){\ThTh_{-i}^{(p)}} \| +\|\CB_{1,i}(\ThTh_{-i})\| + \beta_{1,i}(\ThTh_{-i})\alpha_2 \left( 3{C^{*3}}\right) \}\nonumber\\
        \leq&  32 C^* K \log(nK)(Tnq_n)^{1/2} + \sqrt{2T} (Tnq_n)^{1/2}\|\hat{\ThTh} - \ThTh^* \|_{F} + {C^*}^2 \sqrt{T}\sqrt{2Tnq_n}\|\hat{\ThTh} - \ThTh^* \|_{F} \nonumber\\
        &+ {C^{*3}} T \|\hat{\ThTh} - \ThTh^* \|_{F}^2\nonumber \\
        \lesssim& K\log(nK)(Tnq_n)^{1/2} +  \sqrt{T} (Tnq_n)^{1/2}\sqrt{nC_n}(p_n n)^{-1/4}  + T C_nn(p_n n)^{-1/2}\nonumber\\
        \asymp& (nq_n)^{1/2} + n^{1/2}(nq_n)^{1/4} + n^{1/2}q_n^{-1/2} \nonumber\\
        \leq & n^{1/2}(nq_n)^{1/4}\nonumber\\
        \leq& nq_n
    \end{align}
    for $n$ large enough as $q_n \gtrsim n^{-1/3}\log(n)$ by Assumption \ref{assp: p_n with rank}. 
    Next, we find a lower bound for $\sigma_{2K}(\CI_{1,i}(\hat{\ThTh}_{-i}) )$. Note that $\sigma^2_{2K}(\ThTh^*) \asymp n$ and $\|\ThTh^*\|_{2 \to \infty} \leq C^*$ by Assumption \ref{assp: rank}. Hence we have $p_n\sigma^2_{2K}(\ThTh^*) \geq 32 \|\ThTh^*\|_{2 \to \infty}\log(2Kn)$ for $n$ large enough. According to Lemma \ref{lm: lower bound of I}, with probability at least $1 - 1/(2nK)$, 
    \begin{align*}
        \min_{i \in [n]} \sigma_{2K}^2( \text{diag}(\Omega_i){\ThTh_{-i}^{*(p)}} )\geq 0.5 p_n \sigma_{2K}^2(\ThTh^*).
    \end{align*}
    Also, we can verify that 
    \begin{align*}
       \max_{i \in [n]} \|\text{diag}(\Omega_i)({\hat{\ThTh}_{-i}^{(p)}} - {\ThTh_{-i}^{*(p)}})\|_{2}^2 &\leq \|\hat{\ThTh}^{(p)} - {\ThTh^{*(p)}} \|_{F}^2 \\
       &\lesssim n C_n (p_n n)^{-1/2}\\
       &\leq 0.5\min_{i \in [n]} \sigma_{2K}^2( \text{diag}(\Omega_i){\ThTh_{-i}^{*(p)}} )
    \end{align*}
    under the assumption that $p_n \gtrsim n^{-1/3}\log(n)$. Therefore, by Lemma \ref{lm: lower bound of I final}, with probability converging to $1$, we have 
    \begin{align}\label{eq: sigma I bound}
         \min_{i \in [n]}\sigma_{2K}(\CI_{1,i}(\hat{\ThTh}_{-i})) \geq \delta_9 p_nn,
    \end{align}
    for some constant $\delta_9>0$. Next, we verify conditions of Lemma \ref{lm: lemma 16 in chen jmlr}. According to Lemma \ref{lm: bound for gamma1i}, we have
    \begin{align*}
        \max_{i\in [n]}\gamma_{1,i}(\hat{\ThTh}_{-i}) \leq 2Tn q_n {C^{*3}} \lesssim nq_n
    \end{align*}
    for $n$ large enough. Therefore, with probability tending to $1$, we have 
    \begin{align} \label{eq: gamma lower bound}
        \min_{i\in [n]} \gamma_{1,i}(\ThTh_{-i})^{-1}(\alpha_2 \left( 3{C^{*3}}\right))^{-1} \sigma^2_{2K}(\CI_{1,i}(\ThTh_{-i}) ) \gtrsim (nq_n)^{-1}(p_nn)^2 \asymp nq_n.
    \end{align}
    Therefore, by \eqref{eq: max z B beta} and \eqref{eq: gamma lower bound}, we have
    \begin{align}
        &\max_{i \in [n]} \{\|\zz_i  \text{diag}(\Omega_i){\hat{\ThTh}_{-i}^{(p)}} \| +\|\CB_{1,i}(\hat{\ThTh}_{-i})\| + \beta_{1,i}(\hat{\ThTh}_{-i})\alpha_2 \left( 3{C^{*3}}\right) \} \nonumber \\
        \leq& \min_{i\in [n]} 0.25 \gamma_{1,i}(\hat{\ThTh}_{-i})^{-1}(\alpha_2 \left( 3{C^{*3}}\right))^{-1} \sigma^2_{2K}(\CI_{1,i}(\hat{\ThTh}_{-i}) ) 
    \end{align}
    with probability converging to $1$. Similarly, according to \ref{eq: sigma I bound}, we have 
       \begin{align}
        &\max_{i \in [n]} \{\|\zz_i  \text{diag}(\Omega_i){\hat{\ThTh}_{-i}^{(p)}} \| +\|\CB_{1,i}(\hat{\ThTh}_{-i})\| + \beta_{1,i}(\hat{\ThTh}_{-i})\alpha_2 \left( 3{C^{*3}}\right) \}\nonumber  \\
        \leq&   0.5 \sigma_{2K}(\CI_{1,i}(\hat{\ThTh}_{-i}) )C^*
    \end{align}
with probability converging to $1$. Thus, conditions of Lemma \ref{lm: lemma 16 in chen jmlr} are satisfied. By \eqref{eq: max z B beta} and \eqref{eq: sigma I bound} we have $\|\tilde{\ThTh} - \ThTh^*\|_{ 2\to \infty}\leq C^*$ and 
\begin{align*}
    \|\tilde{\ThTh} - \ThTh^*\|_{ 2\to \infty} \lesssim  n^{1/2}(nq_n)^{1/4} (np_n)^{-1}  = n^{-1/4}p_n^{-3/4}. 
\end{align*}
Moreover, from \eqref{eq: sigma I bound}, the optimization problem $\max_{\thth_i \in \RR^{2K}} \sum_{j \in [n], j \neq i} \{ y_{ij}\log( g( \thth^{\top}_i \hat{\thth}_j^{(p)} )) + (n_{ij}-y_{ij})\log( 1 - g( \thth^{\top}_i \hat{\thth}_j^{(p)})) \}$ is strictly convex. Thus, $\tilde{\thth}_i$ is the unique solution to this optimization problem, and the proof is complete. \end{proof}

With the above lemma, we have shown that \eqref{eq: tilde Th convergence} in Theorem \ref{thm: 2 to infty} holds. It only remains to prove that \eqref{eq: tilde M convergence} holds. Note that $\tilde{M} - M^* = \tilde{\ThTh}J_K \tilde{\ThTh}^{\top} - \ThTh^* J_K {\ThTh^*}^{\top} = (\tilde{\ThTh} - \ThTh^*)J_K\tilde{\ThTh}^{\top} + \ThTh^* J_K(\tilde{\ThTh} - \ThTh^*)^{\top}$. Hence, 
\begin{align*}
     \|\tilde{M} - M^* \|_{\infty} \leq 2\|\tilde{\ThTh} - \ThTh^* \|_{2 \to \infty} \|J_K\tilde{\ThTh}^{\top}\|_{2 \to \infty} .
\end{align*}
Therefore, by \eqref{eq: tilde Th convergence} and $\|J_K\tilde{\ThTh}^{\top}\|_{2 \to \infty} \leq C^* $, we have 
\begin{align*}
     \|\tilde{M} - M^* \|_{\infty}  \lesssim n^{-1/4}p_n^{-3/4}
\end{align*}
with probability converging to $1$, and the proof of Theorem \ref{thm: 2 to infty} is complete.
 }

\subsection{\rev{Proof of Theorem \ref{thm: top k recovery}}}\label{app: proves of thm: top k recovery}
\rev{
\begin{proof}
For ease of presentation, assume $\CS = [n]$ and that the players are ordered such that
$r_1(\Pi^*) \ge r_2(\Pi^*) \ge \cdots \ge r_n(\Pi^*)$ without loss of generality. We write $\Delta_k$ for $\Delta_{k,\CS}$ and $r_i(\Pi)$ for $r_{i,\CS}(\Pi)$. It suffices to show that
\[
\min_{i \in [k]} r_i(\tilde{\Pi})
>
\max_{l \in \{k+1, \dots, n\}} r_l(\tilde{\Pi})
\]
with probability converging to $1$. For any $i \in [n]$, we have
\begin{align*}
    |r_i(\tilde{\Pi}) - r_i(\Pi^*)|
    = \left| \frac{1}{n} \sum_{j=1}^n \big( \tilde{\Pi}_{ij} - \Pi^*_{ij} \big) \right| 
    \le \|\tilde{\Pi} - \Pi^*\|_{\max}.
\end{align*}
Hence
\[
\max_{i \in [n]} |r_i(\tilde{\Pi}) - r_i(\Pi^*)|
\le
\|\tilde{\Pi} - \Pi^*\|_{\max}.
\]

Consequently, for $i \in [k]$ and $l \in \{k+1, \dots, n\}$,
\[
r_i(\tilde{\Pi})
\ge
r_i(\Pi^*) - \|\tilde{\Pi} - \Pi^*\|_{\max},
\qquad
r_l(\tilde{\Pi})
\le
r_l(\Pi^*) + \|\tilde{\Pi} - \Pi^*\|_{\max}.
\]

Recall that by \eqref{eq: Pi max norm convergence}, we have
\[
\|\tilde{\Pi} - \Pi^*\|_{\infty}
\le
\kappa_8 n^{-1/4} p_n^{-3/4}
\]
with probability converging to $1$. On this event, since
$
\Delta_k
>
2 \kappa_8 n^{-1/4} p_n^{-3/4},
$ by assumption, we have 
then
$
\Delta_k
>
2 \|\tilde{\Pi} - \Pi^*\|_{\max}.
$

Using $r_k(\Pi^*) = r_{k+1}(\Pi^*) + \Delta_k$, we obtain
\begin{align*}
\min_{i \in [k]} r_i(\tilde{\Pi})
&\ge
r_k(\Pi^*) - \|\tilde{\Pi} - \Pi^*\|_{\max} \\
&=
r_{k+1}(\Pi^*) + \Delta_k - \|\tilde{\Pi} - \Pi^*\|_{\max} \\
&>
r_{k+1}(\Pi^*) + \|\tilde{\Pi} - \Pi^*\|_{\max} \\
&=
\max_{l \in \{k+1, \dots, n\}}
\big( r_l(\Pi^*) + \|\tilde{\Pi} - \Pi^*\|_{\max} \big) \\
&\ge
\max_{l \in \{k+1, \dots, n\}} r_l(\tilde{\Pi}).
\end{align*}

Therefore, the top-$k$ set based on $\tilde{\Pi}$ coincides with that based on $\Pi^*$ with probability converging to $1$, which completes the proof.\end{proof}}

\subsection{\rev{Proof of Proposition \ref{prop: skew preserving}}}\label{app: proof of proposition}
\begin{proof}
    We consider the case where $n$ is even; the proof for odd $n$ is analogous. It is well known that $M$ can be decomposed in the Murnaghan canonical form $M = Q X Q ^{\top}$ \citep{murnaghan1931canonical,benner_etal-2000}, where $Q$ is orthogonal and $X$ is block-diagonal of the form 
\begin{align*}
X =  
\begin{pmatrix}
       0 & \sigma_1 & 0 & 0 & \dots & 0 & 0 \\
    -\sigma_1 & 0 & 0 & 0 & \dots & 0 & 0 \\
    0 & 0 & 0 & \sigma_2 & \dots & 0 & 0 \\
    0 & 0 & -\sigma_2 & 0 & \dots & 0 & 0 \\
    \vdots & \vdots & \vdots & \vdots & \ddots & \vdots & \vdots \\
    0 & 0 & 0 & 0 & \dots & 0 & \sigma_{n/2} \\
    0 & 0 & 0 & 0 & \dots & -\sigma_{n/2} & 0 
\end{pmatrix},
\end{align*}
where $\sigma_1, \dots, \sigma_{n/2}$ are the singular values of $M$. It can be verified that the projection operator $P_{\tau}(M)$ preserves the Murnaghan canonical form as $\PP_{\tau}(M) = Q YQ^{\top}$, where 
\begin{align*}
Y =  
\begin{pmatrix}
    0 & \max\{\sigma_1 - \lambda, 0\} & 0 & \dots & 0 \\
    -\max\{\sigma_1 - \lambda, 0\} & 0 & 0 & \dots & 0 \\
    0 & 0 & \ddots & \dots & 0 \\
    \vdots & \vdots & \vdots & 0& \max\{\sigma_{n/2} - \lambda, 0\} \\
    0 & 0 & 0 & -\max\{\sigma_{n/2} - \lambda, 0\} & 0
\end{pmatrix}.
\end{align*}
Hence we have $P_{\tau}(M) \in \text{Skew}_n$. 
\end{proof}

\section{\rev{Additional Simulation Results}}\label{app: Additional Simulation Results}
\rev{In this section, we present additional simulation results, including the computation time and out-of-sample predictive performances.}

\rev{The computational time of the proposed method and the BT model is summarized in Table~\ref{tab: computational time}. As expected, the BT model is consistently faster than the proposed method due to its simpler structure. While the computational time increases with $n$ for both methods, it is also notable that the running time of the proposed method increases with the sparsity level and the latent dimension. This reflects an increase in computational complexity of the estimation problem when the underlying structure becomes more difficult to recover.}

\begin{table}[]
\centering
    \small
    \rev{
\begin{tabular}{llllllll}
\toprule
   &      & \multicolumn{2}{c}{Sparse} & \multicolumn{2}{c}{Less Sparse} & \multicolumn{2}{c}{Dense} \\
\midrule
$k$  & $n$    & proposed       & BT        & proposed         & BT           & proposed      & BT        \\
\midrule
1  & 500  & 16.63          & 0.08      & 11.81            & 0.08         & 14.61         & 0.14      \\
2  & 500  & 25.15          & 0.05      & 16.81            & 0.08         & 13.97         & 0.14      \\
3  & 500  & 35.34          & 0.05      & 23.05            & 0.08         & 16.21         & 0.14      \\
4  & 500  & 48.52          & 0.05      & 24.92            & 0.08         & 16.61         & 0.13      \\
5  & 500  & 66.45          & 0.05      & 30.14            & 0.08         & 24.84         & 0.13      \\
6  & 500  & 96.33          & 0.05      & 36.54            & 0.08         & 17.38         & 0.13      \\
7  & 500  & 136.75         & 0.05      & 44.27            & 0.08         & 20.43         & 0.13      \\
8  & 500  & 158.26         & 0.05      & 55.45            & 0.08         & 30.61         & 0.13      \\
9  & 500  & 176.87         & 0.05      & 68.04            & 0.08         & 32.31         & 0.13      \\
10 & 500  & 201.67         & 0.05      & 83.15            & 0.08         & 33.41         & 0.12      \\
\hline
1  & 1000 & 125.86         & 0.22      & 76.32            & 0.25         & 86.25         & 0.69      \\
2  & 1000 & 195.66         & 0.16      & 119.11           & 0.26         & 106.37        & 0.61      \\
3  & 1000 & 287.85         & 0.16      & 141.20           & 0.26         & 121.19        & 0.60      \\
4  & 1000 & 337.10         & 0.16      & 172.06           & 0.25         & 129.95        & 0.58      \\
5  & 1000 & 497.85         & 0.16      & 218.85           & 0.26         & 123.24        & 0.60      \\
6  & 1000 & 717.48         & 0.16      & 248.50           & 0.26         & 179.03        & 0.57      \\
7  & 1000 & 954.86         & 0.16      & 256.35           & 0.25         & 180.23        & 0.60      \\
8  & 1000 & 1199.15        & 0.16      & 325.61           & 0.25         & 176.08        & 0.56      \\
9  & 1000 & 1477.26        & 0.15      & 405.13           & 0.26         & 204.43        & 0.57      \\
10 & 1000 & 1666.11        & 0.15      & 488.08           & 0.25         & 212.17        & 0.56      \\
\hline
1  & 1500 & 439.61         & 0.45      & 236.00           & 0.68         & 354.05        & 1.52      \\
2  & 1500 & 658.29         & 0.40      & 368.61           & 0.67         & 391.12        & 1.55      \\
3  & 1500 & 893.69         & 0.41      & 471.67           & 0.62         & 412.78        & 1.42      \\
4  & 1500 & 1156.06        & 0.38      & 433.54           & 0.70         & 380.19        & 1.39      \\
5  & 1500 & 1710.13        & 0.42      & 673.19           & 0.60         & 401.08        & 1.39      \\
6  & 1500 & 2321.63        & 0.41      & 708.65           & 0.63         & 437.05        & 1.38      \\
7  & 1500 & 3071.95        & 0.38      & 833.68           & 0.67         & 563.98        & 1.38      \\
8  & 1500 & 3835.15        & 0.40      & 1022.79          & 0.62         & 661.77        & 1.35      \\
9  & 1500 & 5032.98        & 0.41      & 1132.81          & 0.65         & 270.44        & 1.41      \\
10 & 1500 & 5317.70        & 0.40      & 1388.54          & 0.64         & 709.70        & 1.36      \\
\hline
1  & 2000 & 1054.74        & 0.86      & 548.95           & 1.09         & 725.27        & 2.89      \\
2  & 2000 & 1383.09        & 0.83      & 874.44           & 1.13         & 885.54        & 3.09      \\
3  & 2000 & 2046.04        & 0.82      & 1077.11          & 1.13         & 383.48        & 2.86      \\
4  & 2000 & 2663.27        & 0.80      & 806.64           & 1.12         & 939.02        & 2.76      \\
5  & 2000 & 4020.82        & 0.82      & 1455.05          & 1.13         & 934.34        & 2.80      \\
6  & 2000 & 5335.95        & 0.83      & 1569.47          & 1.12         & 990.44        & 2.77      \\
7  & 2000 & 7366.32        & 0.82      & 1957.48          & 1.11         & 1342.22       & 2.76      \\
8  & 2000 & 10153.55       & 0.83      & 2086.00          & 1.14         & 1437.54       & 2.68      \\
9  & 2000 & 12119.82       & 0.82      & 2410.81          & 1.11         & 1109.65       & 2.69      \\
10 & 2000 & 14110.18       & 0.86      & 2873.36          & 1.10         & 1335.29       & 2.65     \\
\bottomrule
\end{tabular}
}
\label{tab: computational time}
    \caption{\rev{Total computational time (in minutes) for the proposed method and the Bradley--Terry (BT) model across different sparsity levels (sparse, less sparse, dense), sample sizes ($n=500,1000,1500,2000$), and rank parameters $k=1,\ldots,10$. Each value represents the total time required to complete 50 simulation runs.}}
\end{table}

\rev{We also present results evaluating the out-of-sample predictive performances. Specifically, for each simulation replicate, after generating the training data and estimating the model parameters, we independently generate a test set from the same underlying comparison probability matrix $\Pi^*$. To mimic a realistic data-analysis scenario, we control the size of the test set to be approximately 25\% of the training data.}

\rev{In particular, for each pair $(i,j)$, we independently generate a potential number of test comparisons $\tilde{n}_{ij}$ using the same sampling mechanism as in the training phase. We then obtain the actual number of comparisons by generating $n^{\text{(test)}}_{ij}$ following $ \mathrm{Bin}(\tilde{n}_{ij}, 0.25)$. The comparison outcomes are then independently generated according to the true probability $\pi^*_{ij}$ conditional on $n^{\text{test}}_{ij}$. We let $Y^{\text{(test)}} = ( y_{ij}^{(\txt)})_{n \times n}$ denote the $n \times n $ matrix collecting all test comparison outcomes.}

\rev{We compute the normalized predictive likelihood given by 
\begin{align}\label{eq: lik loss}
    L(Y^{(test)} \mid \hat{\Pi}) =  \frac{1}{ \sum_{i=1}^{n} \sum_{j > i} n_{ij} } \sum_{i=1}^{n} \sum_{j > i} \left( y_{ij}^{(\txt)} \log( \hat{\pi}_{ij} ) + y_{ji}^{(\txt)} \log( 1- \hat{\pi}_{ij} ) \right),  
\end{align}
and the average predictive likelihood across 50 simulations for each model is reported in Figure \ref{fig:lik comparison}, considering different values of $n$, $k$, and sparsity levels.}

\begin{figure}[t]
    \centering
    \includegraphics[width=\linewidth]{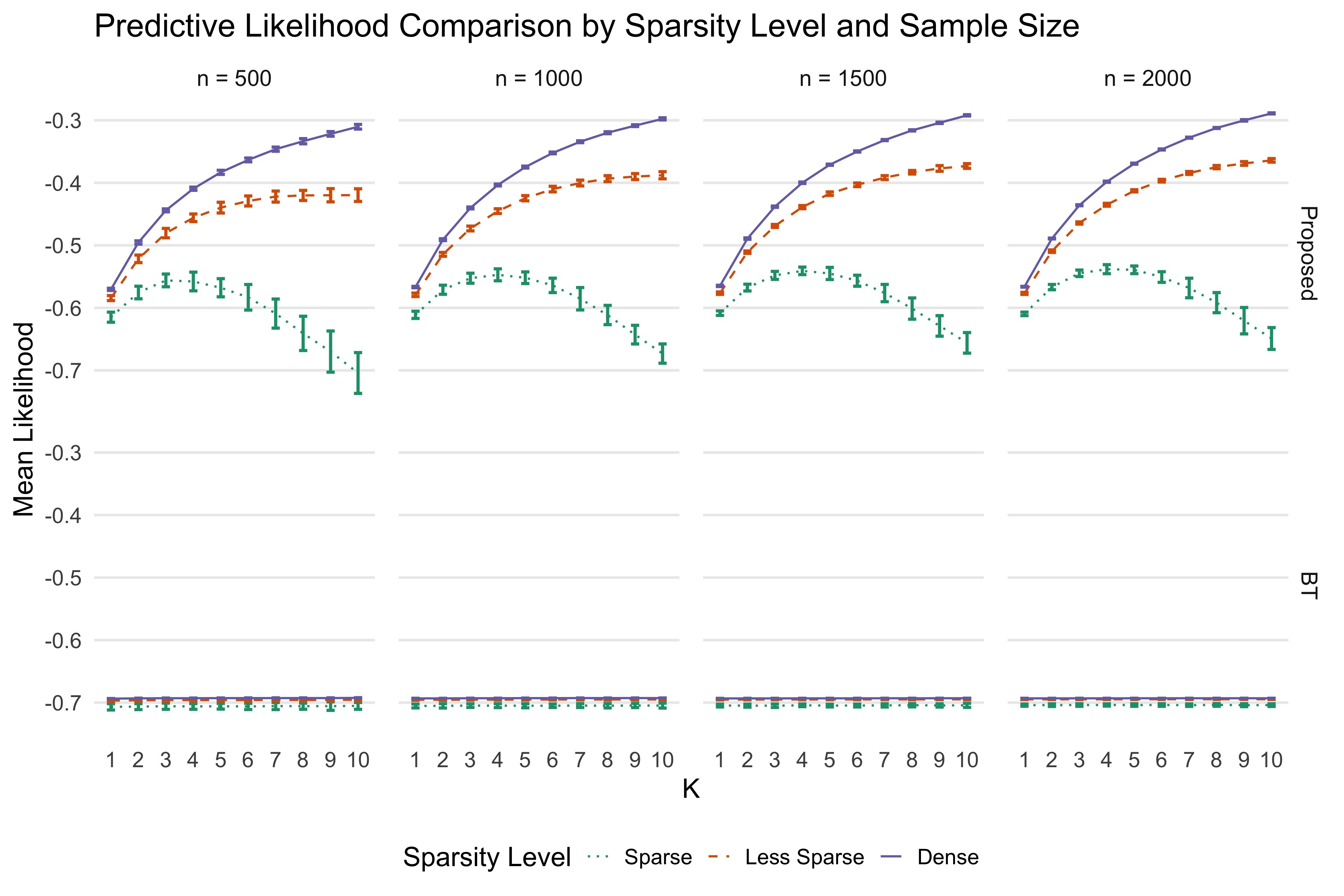}
    \caption{\rev{Comparison of predictive likelihood between the proposed method and the Bradley-Terry (BT) model across different sparsity levels (sparse, less sparse, dense). Each row corresponds to a method (top: proposed; bottom: BT), and each column corresponds to a sample size ($n = 500, 1000, 1500, 2000$). The x-axis represents the rank parameter \( k \), while the y-axis shows the mean likelihood, computed as the average of the likelihood defined in \eqref{eq: lik loss}.Error bars indicate $\pm 2$ standard deviations across replicates.}}
    \label{fig:lik comparison}
\end{figure}

\bibliography{reference} 
\end{document}